\documentclass{article}

 \usepackage[preprint]{neurips_2026}

\usepackage[utf8]{inputenc} 
\usepackage[T1]{fontenc}    
\usepackage{hyperref}       
\usepackage{url}            
\usepackage{booktabs}       
\usepackage{amsfonts}       
\usepackage{nicefrac}       
\usepackage{microtype}      
\usepackage{xcolor}         

\usepackage{rotating}
\usepackage{amsmath,amsthm,amssymb}
\usepackage{thmtools}
\usepackage{enumerate}
\usepackage{enumitem}
\usepackage{bm}
\usepackage{adjustbox}
\usepackage{subcaption}
\usepackage{wrapfig}
\usepackage{multicol}
\usepackage{multirow}
\usepackage{algorithm}
\usepackage{algorithmic}
\usepackage{kotex}
\usepackage{tabularx}
\usepackage[capitalize,noabbrev]{cleveref}
\Crefname{appsec}{Appendix}{Appendices}
\crefname{appsec}{Appendix}{Appendices}
\Crefname{appsubsec}{Appendix}{Appendices}
\crefname{appsubsec}{Appendix}{Appendices}
\Crefname{appsubsubsec}{Appendix}{Appendices}
\crefname{appsubsubsec}{Appendix}{Appendices}
\crefname{equation}{eq.}{eqs.}
\Crefname{equation}{Eq.}{Eqs.}

\usepackage{amsmath,amsfonts,bm}

\def\eqref#1{equation~\ref{#1}}

\def\1{\bm{1}}

\DeclareMathAlphabet{\mathsfit}{\encodingdefault}{\sfdefault}{m}{sl}
\SetMathAlphabet{\mathsfit}{bold}{\encodingdefault}{\sfdefault}{bx}{n}

\DeclareMathOperator*{\argmin}{arg\,min}

\newtheorem{theorem}{Theorem}[section]

\newtheorem{lemma}[theorem]{Lemma}
\newtheorem{proposition}[theorem]{Proposition}
\newtheorem{remark}{Remark}[section]

\newtheorem{corollary}[theorem]{Corollary}
\newtheorem{definition}{Definition}[section]

\newtheorem{assumption}{Assumption}[section]

\definecolor{cadmiumgreen}{rgb}{0.1, 0.42, 0.24}
\definecolor{cornellred}{rgb}{0.80, 0.11, 0.11}
\definecolor{Gray}{gray}{0.9}

\hypersetup{citecolor=MidnightBlue, linkcolor=BrickRed}
\hypersetup{
  linkcolor = cornellred,
  citecolor  = cadmiumgreen,
  colorlinks = true,
}

\title{What Drives the Inlier-Memorization Effect?\\
A Theory of Outlier Detection via Early Training Dynamics}

\author{%
    Kunwoong Kim \\
    KAIST AI \\
    \texttt{kunwoong.kim@kaist.ac.kr}
    \\
    \And
    Dongha Kim\thanks{Corresponding author.} \\
    Department of Statistics \\
    Center for Data Science \\
    Sungshin Women’s University \\
    \texttt{dongha0718@sungshin.ac.kr}
}

\begin{document}

\maketitle

\begin{abstract}
    Outlier detection (OD) aims to identify anomalous instances by learning the underlying structure of normal data (inliers), and is particularly challenging in fully unsupervised settings where no information about anomalies is available during training. 
    Recent advances have leveraged the \textit{inlier-memorization (IM) effect}, a phenomenon in which deep models memorize inlier patterns earlier than those of outliers, as a powerful signal for distinguishing outliers. 
    However, despite its empirical success, the theoretical understanding of the IM effect remains limited.
    In this work, we present a theoretical study of the IM effect. 
    Focusing on a simple autoencoder, we show that, under mild assumptions, the model can successfully memorize inliers while failing to memorize outliers during certain stages of early training. 
    In particular, we characterize not only the emergence of the IM effect, but also its strength and persistence, and analyze how these properties depend on the data distribution and parameter initialization.
    In addition, building on these insights, we derive simple yet practical guidelines for enhancing the IM effect, including data preprocessing and parameter initialization schemes, achieving state-of-the-art performance on the \texttt{ADBench} datasets. 
    Our findings provide a theoretical foundation for the IM effect and offer actionable directions for improving IM-based outlier detection methods.
\end{abstract}

\section{Introduction}
\label{sec:intro}

Outlier detection (OD) aims to identify anomalous instances (outliers) that deviate from the underlying structure of normal data (inliers) \citep{chandola2009anomaly,chalapathy2019deep}.
It plays a crucial role in a wide range of applications, including fraud detection \citep{hilal2022financial}, cybersecurity \citep{ahmed2016survey}, medical diagnosis \citep{fernando2021deep}, and data preprocessing for downstream learning tasks.
Among various settings, unsupervised outlier detection (UOD), where both inliers and outliers are contained in the training data and no information about anomalousness is accessible, is the most general and particularly challenging setting. 
In this scenario, the model must learn the patterns of inliers while being trained on contaminated data.

To this end, a recent line of work exploits the \textit{inlier-memorization (IM) effect} \citep{DBLP:conf/icml/KimHLKK24}: when a deep unsupervised model is trained on data containing both inliers and outliers, it prioritizes reducing the loss of inliers before that of outliers in early training steps.
Thus, the per-sample loss of an under-fitted model effectively serves as an outlier score, and this insight motivates various methodologies \citep{DBLP:conf/icml/KimHLKK24,zhang2024gradstop,DBLP:conf/aaai/ChoHBK25,kang2026memorize}.
These approaches consistently demonstrate strong empirical performance across diverse datasets, highlighting the practical importance of the IM effect in UOD.

Despite these successes, the IM effect itself remains largely underexplored from a theoretical perspective.
Existing works primarily rely on empirical observations or heuristic arguments (for example, attributing the IM effect to the higher density and prevalence of inliers in the data distribution), but lack a rigorous understanding of how long the effect persists and how strong the resulting separation is, i.e., its persistence and strength, and how they depend on key factors such as the data distribution and model initialization.
This gap limits both the interpretability of the IM effect and the principled design of more reliable IM-based outlier detection methods.

In this work, we theoretically study the IM effect. 
For a simple autoencoder trained by minimizing reconstruction error via gradient descent, we theoretically show that, under mild assumptions, an under-fitted autoencoder memorizes inliers while failing to memorize outliers.
We characterize the conditions under which the IM effect exhibits persistence (i.e., the range of training steps over which the effect holds) and strength (reflecting how many outliers can be separated from inliers), from two key perspectives: the data distribution and weight initialization.
In particular, we show that the IM effect persists longer and exhibits stronger separability when inliers form dense clusters and outliers are sparse and well-separated, and that it further persists longer when the initialization captures inlier structure.
\cref{fig:im_effect} illustrates the IM effect and the factors governing its strength and persistence.

\begin{figure}[t]
    \centering
    \includegraphics[width=0.95\linewidth]{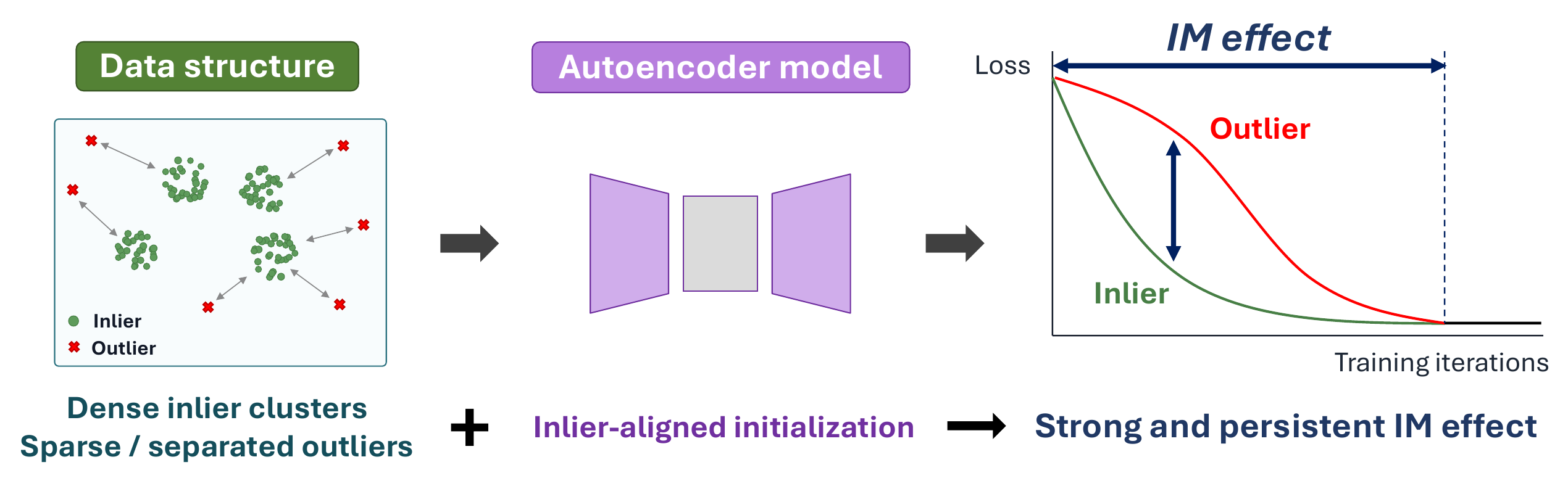}
    \caption{
    Inlier-memorization (IM) effect and the factors that influence its strength (how large the loss gap is) and persistence (how long the loss gap persists).
    }
    \label{fig:im_effect}
    \vskip -0.2in
\end{figure}

Building on these insights, we provide two simple yet effective guidelines for enhancing the IM effect in practice.
First, we show that using embedding representations can induce more favorable data structures, making the IM effect clearer.
Second, we demonstrate that applying EMA-based weight initialization helps preserve inlier information during early training, allowing the IM effect to last longer.
These guidelines are intentionally simple and broadly applicable, requiring only minimal modifications to existing methods. 
We empirically verify that incorporating these two components into existing IM-based outlier detection approaches consistently leads to improved performance, achieving gains over state-of-the-art IM-based methods.



The remainder of this paper is organized as follows. 
Section \ref{sec:review} briefly reviews related research on OD problems. 
In Section \ref{sec:main}, we present a theoretical study of the IM effect, including the problem setup, a formal description of the phenomenon, and the corresponding theoretical analysis.
Section \ref{sec:exp} provides empirical validation and discusses practical implications, including guidelines for data preprocessing and weight initialization, followed by concluding remarks in Section \ref{sec:conclusion}.

Our contributions can be summarized as follows:

\begin{itemize}[topsep=0pt, leftmargin=1.5em, labelsep=0.5em]
    \item[$\diamond$] We theoretically study the inlier-memorization (IM) effect, showing that under-fitted autoencoders memorize inliers earlier than outliers, and identifying key factors that govern this behavior, including the data distribution and weight initialization.
    \item[$\diamond$] Inspired by our theoretical study, we propose two simple yet effective strategies for IM-based outlier detection and show that they consistently improve performance over existing IM-based methods, achieving state-of-the-art results.
\end{itemize}

\section{Related Works}
\label{sec:review}

Recent advances have shown that deep learning provides a powerful framework for UOD.
Classical approaches include density-based methods such as LOF \citep{10.1145/335191.335388}, cluster-based methods such as CBLOF \citep{he2003discovering}, robust statistics-based approaches such as MCD \citep{fauconnier2009outliers}, and tree-based methods such as Isolation Forest (IF) \citep{liu2008isolation}.
Boundary-based methods such as OCSVM \citep{ocsvm} and SVDD \citep{svdd} learn a separating boundary for normal data, with deep extensions including DeepSVDD \citep{deepsvdd} and DeepSAD \citep{deepsad}.


Self-supervised learning has also been widely applied to OD tasks \citep{csi,golan2018deep}.
Representative approaches include SimCLR \citep{simclr}, which learns informative representations via contrastive learning, and ICL \citep{icl}, which detects anomalies by maximizing mutual information between masked and unmasked data.
NeuTraL AD \citep{DBLP:conf/icml/QiuPKMR21} further learns neural transformations to generate diverse views, enabling effective anomaly detection. 


A wide range of deep learning techniques with diverse modeling perspectives have been applied to UOD.
For example, RDA \citep{rda} enhances autoencoder-based anomaly detection by improving robustness to outliers during representation learning, while DSEBM \citep{zhai2016deep} detects anomalies using deep energy-based models that assign low energy to normal samples and high energy to anomalous ones.
AnoGAN \citep{DBLP:conf/ipmi/SchleglSWSL17} employs generative adversarial networks to learn the manifold of normal data, and GAAL \citep{liu2019generative} generates informative potential outliers to better define the decision boundary.
More recently, DTE \citep{DBLP:journals/corr/abs-2305-18593} leverages diffusion-based models for likelihood-based detection, and ROBOD \citep{DBLP:conf/nips/DingZA22} improves robustness via a hyper-ensemble framework.


Recently, the IM effect has attracted increasing attention in the outlier detection literature.
Following its introduction in ODIM \citep{DBLP:conf/icml/KimHLKK24}, several subsequent works have developed improved methods based on this phenomenon.
For instance, ALTBI \citep{DBLP:conf/aaai/ChoHBK25} enhances the IM effect via adaptive mini-batch scheduling and trimmed loss optimization, while IMBoost \citep{kang2026memorize} improves performance by incorporating active learning.
These methods achieve strong performance while remaining efficient.
However, they largely assume the existence of the IM effect, without providing a theoretical explanation of why it arises or when it can be expected to hold.
Although GradStop \citep{zhang2024gradstop} takes a step toward understanding this phenomenon, a rigorous theoretical characterization of the IM effect remains limited.


\section{A Theoretical Analysis of the Inlier-Memorization Effect}
\label{sec:main}


\subsection{Preliminaries and Problem Setup}
\label{sec:notation}

\paragraph{Inlier-Memorization Effect}
When training deep unsupervised models on datasets containing both inliers and outliers, it is often observed that the model tends to fit inlier samples earlier than outliers \citep{DBLP:conf/icml/KimHLKK24,zhang2024gradstop}.
Consequently, after a certain number of training iterations, the reconstruction or likelihood errors of inliers become small, whereas those of outliers remain relatively large. 
We refer to this phenomenon as the \emph{inlier-memorization (IM) effect}. 
Intuitively, inliers are more frequent and share a common structure, whereas outliers are less frequent and deviate from this structure, making them harder to fit.
As a result, reducing the loss of inliers first provides a more effective direction for minimizing the overall loss during the early stages of training, thereby leading to the IM effect.

Importantly, the IM effect is a general phenomenon arising during the training of deep unsupervised models, rather than being specific to a particular architecture.
It has been empirically observed across a range of models, including autoencoders, variational autoencoders (VAEs, \citep{kingma2013auto,pmlr-v32-rezende14}), and normalizing flows \citep{kobyzev2020normalizing}.
For theoretical simplicity, we focus our analysis on a simple autoencoder.
Extending the analysis to other generative models is a promising direction for future work.




\paragraph{Data Distribution}

Under the UOD regime, we observe data $\{x_i\}_{i=1}^n\subset \mathbb{R}^p$ containing both inliers and outliers. 
For simplicity, we assume that the norm of each sample is uniformly bounded as 
$\|x_i\|_2 \le 1$ for all $i\in[n]$. 
In practice, this condition can be readily satisfied, e.g., by applying min-max normalization followed by scaling.

We assume that the data exhibit an underlying \textit{clustered structure}. 
Specifically, there exist $K$ cluster centers $\mu_1,\dots,\mu_K \in \mathbb{R}^p$. 
Let $g:[n]\to[K]$ denote a cluster-assignment map, where $[N]:=\{1,\ldots,N\}$ for $N\in\mathbb{N}$. 
We represent each observed sample by a decomposition
\begin{equation}\label{eq:cluster_struc}
    x_i = \widetilde{x}_i + \xi_i + o_i, \quad i \in [n],
\end{equation}
where
(i) $\widetilde{x}_i := \mu_{g(i)}$ is the cluster center associated with $x_i,$
(ii) $\xi_i\in\mathbb{R}^p$ represents within-cluster variation,
and 
(iii) $o_i\in\mathbb{R}^p$ denotes an additional perturbation corresponding to outliers. 
For each cluster $k\in[K]$, we define $\Lambda_k := \{i\in[n]: g(i)=k\}$ as the index set of data assigned to $k$-th cluster.
We further define $n_{\min} := \min_{k \in [K]} n_k$ and $n_{\max} := \max_{k \in [K]} n_k$. 
And let $\delta := \min_{a \neq b \in [K]} \|\mu_a - \mu_b\|_2$ denote the minimum separation between cluster centers.
%

The term $\xi_i$ captures small deviations of samples around the cluster center, reflecting the natural noise of inlier data.
\cref{assumption:clustered_model} assumes that inlier data remain close to their corresponding cluster centers.

\begin{assumption}[Small within-cluster variation]\label{assumption:clustered_model}
    $\|\xi_i\|_2 \le \varepsilon_c, \forall i\in[n],$ for some constant $\varepsilon_c \ge 0$.
\end{assumption}


Furthermore, the term $o_i$ represents abnormal deviations that generate outliers.
Under this decomposition, we define $x_i$ as an \textit{inlier} when $o_i=0$, and as an \textit{outlier} otherwise. 
Let $\mathcal{O}:=\{i\in[n]:o_i\neq0\}$ denote the index set of outliers.
The {outlier rate} of cluster $k$ is then defined as $\frac{|\mathcal{O}\cap\Lambda_k|}{n_k}$.
\cref{assumption:outlier_rate} assumes that the fraction of outliers in each cluster is uniformly bounded.

\begin{assumption}[Bounded outlier rate]\label{assumption:outlier_rate}
    The outlier rate is at most $\rho$ for some constant $\rho>0$, i.e., $\max_{k\in[K]} \frac{|\mathcal{O}\cap\Lambda_k|}{n_k} \le \rho$.
\end{assumption}

In addition, we say that an outlier is \emph{well-separated} if its additional perturbation is relatively large:

\begin{assumption}[Well-separated outlier]\label{assumption:well_separation}
    Let $x_i$ be an outlier with perturbation $o_i$. 
    We say that $x_i$ is well-separated if
    $\|o_i\|_2 > \delta + 2\varepsilon_c.$
\end{assumption}

Here, the quantity $\delta + 2\varepsilon_c$ corresponds to an upper bound on the distance between inliers from two closest clusters. 
Thus, a well-separated outlier lies sufficiently far from all inliers.

\paragraph{Model Setup}
We consider a bias-free single-hidden-layer autoencoder $F_W:\mathbb{R}^p\to\mathbb{R}^p$ defined as
\begin{align}\label{eq:autoencoder}
    F_W(x):=A\phi(Wx),
\end{align}
where $W\in\mathbb{R}^{H\times p}$ is the \textit{trainable} encoder weight matrix ($H$ denotes the number of hidden nodes), $A\in\mathbb{R}^{p\times H}$ is the \textit{fixed} decoder weight matrix, and $\phi$ is an activation function applied elementwise.

Following standard autoencoder training frameworks \cite{doi:10.1126/science.1127647}, we train an autoencoder model to reconstruct each input from itself (i.e., the prediction target is the input itself).
To learn the encoder parameters $W$, we minimize the mean squared reconstruction error over the training data.
Specifically, we minimize the following mean squared reconstruction error
\begin{equation}\label{eq:recon_loss}
    \mathcal L(W) := \frac{1}{2}\sum_{i=1}^{n}\|x_i - F_W(x_i)\|_2^2=\frac{1}{2}\|X-\mathcal{F}_{W}(X) \|_2^2
\end{equation}
where
$X := [x_1^{\top}, \ldots, x_n^{\top}]^{\top} \in \mathbb{R}^{np}$ and $\mathcal{F}_{W}(X) := [F_W(x_1)^{\top}, \ldots, F_W(x_n)^{\top}]^{\top} \in \mathbb{R}^{np}$. 
That is, the model is trained so that $x_i \approx F_W(x_i)$ for all $i\in[n]$.
We optimize this objective using gradient descent, a standard learning approach in deep learning.
For theoretical simplicity, as mentioned above, we assume that only the encoder weight matrix $W$ is trained, while the decoder weight matrix $A$ is fixed during optimization. 
We note that such a simplification is commonly adopted in theoretical analyses of single-hidden-layer neural networks \citep{DBLP:conf/icml/AroraDHLW19,DBLP:journals/corr/abs-1902-04674,DBLP:conf/aistats/LiSO20}.

\paragraph{Training via Gradient Descent}
We minimize the objective function \cref{eq:recon_loss} using gradient descent (GD) with initialization 
$W_0 \in \mathbb{R}^{H\times p}$ and learning rate $\eta>0$. 
Let $w := \mathrm{vec}(W) \in \mathbb{R}^{Hp}$ denote the vectorized parameter, and we use $W$ and $w$ interchangeably when no confusion arises.
The GD algorithm at each iteration step $\tau\in\mathbb{Z}_{\ge 0}$ updates the parameter as 
$ w_{\tau+1}=w_\tau-\eta \nabla_w \mathcal L(W_\tau)$, given an initial weight $w_0$.
The update can be equivalently written as
\begin{align}
\label{eq:gd_alg}  
    w_{\tau+1} = w_\tau-\eta J(W_\tau)^\top r_\tau,
\end{align}
where $J(W_{\tau}) := \frac{\partial}{\partial w} \mathcal{F}_{W_{\tau}}(X) \in \mathbb{R}^{np\times Hp}$ and $r_\tau := \mathcal{F}_{W_\tau}(X) - X\in \mathbb{R}^{np}$ denote the Jacobian of the model outputs with respect to the trainable parameter $w$ and the concatenated residual at iteration step $\tau$, respectively.

\subsection{Theoretical Analysis}
\label{sec:theory}

We define the initialization error on the cluster-centered-inputs 
$
\widetilde{R}_{0}:=\|\mathcal{F}_{W_0}(\widetilde{X})-\widetilde{X}\|_2,
$
where $\widetilde{X} = [\widetilde{x}_1^\top,\ldots,\widetilde{x}_n^\top]^\top$ denotes the concatenation of cluster centers corresponding to the samples. 
This quantity measures how well the initialization $W_0$ captures the underlying cluster structure.
As $W_0$ better captures the inlier cluster structure, $\widetilde{R}_0$ becomes smaller.
In addition, we use the notation $c = c(a_1\uparrow,\ldots,a_m\downarrow)$ to indicate that the constant $c$ depends on $a_1,\ldots,a_m$, where $\uparrow$ and $\downarrow$ denote that $c$ is increasing or decreasing in the corresponding argument. 
Logarithmic factors are ignored in this notation when specifying the monotonicity.

We show that under several mild assumptions, minimizing the reconstruction objective in \cref{eq:recon_loss} via GD in \cref{eq:gd_alg} leads to a training phase in which the reconstruction error of inliers decreases earlier than that of outliers, thereby giving rise to the IM effect. 

\begin{theorem}[IM Effect]\label{thm_main}
    Assume that the training data $\{x_i\}_{i=1}^n$, containing both inliers and outliers, satisfy Assumptions~\ref{assumption:clustered_model} and~\ref{assumption:outlier_rate} as well as regularity conditions in Appendix~\ref{sec-appen:theory}. 
    Suppose that the autoencoder in~\cref{eq:autoencoder} is trained by minimizing the reconstruction error in~\cref{eq:recon_loss} via GD, updating the encoder parameter $W$ according to~\cref{eq:gd_alg}.
    {\color{black}Then, there exist a learning rate $\eta>0$ and two integers $0<T_1<T_2$ with}
    \begin{equation}
        T_1=c(n_{\mathrm{min}}\downarrow) 
        \quad
        \textup{and}
        \quad
        T_2=c(n_{\mathrm{min}}\uparrow, n_{\mathrm{max}}\downarrow,\varepsilon_c\downarrow, \widetilde{R}_{0} \downarrow,\rho\downarrow,\delta\uparrow)
    \end{equation}
    such that for any iteration $\tau\in[T_1,T_2]$, inliers have smaller reconstruction loss than well-separated outliers.
    In particular, for any outlier $x_i$ satisfying \cref{assumption:well_separation}, we have
    \begin{equation}
        \max_{j\in\mathcal{O}^c}\|x_j-F_{W_\tau}(x_j)\|_2^2 < \|x_i-F_{W_\tau}(x_i)\|_2^2.
    \end{equation}
\end{theorem}

The regularity conditions, the full statement of \cref{thm_main}, and its proof are deferred to Appendix~\ref{sec-appen:theory} (see \cref{thm:ae_beyond_exact,cor:ae_outlier_general} for details).
Detailed derivations of the monotonic dependence of $T_1$ and $T_2$ on key factors are provided in \cref{app:dependence}.

When every outlier is well-separated, \cref{thm_main} immediately implies the occurrence of the perfect IM effect, as stated in \cref{cor:perfect_sep}.

\begin{corollary}[Perfect separation via the IM effect]\label{cor:perfect_sep}
    Under the assumptions of Theorem~\ref{thm_main} and with the interval $[T_1,T_2]$ defined therein, suppose that \cref{assumption:well_separation} holds for all $i\in\mathcal{O}$.
    Then for any iteration $\tau\in[T_1,T_2]$, the reconstruction error of every outlier is strictly larger than that of every inlier, i.e.,
    $$
    \max_{j\in\mathcal{O}^c}\|x_j-F_{W_\tau}(x_j)\|_2^2
    <
    \min_{i\in\mathcal{O}}\|x_i-F_{W_\tau}(x_i)\|_2^2.
    $$
    Hence, perfect separation between inliers and outliers in terms of reconstruction loss is achieved during the training iterations $[T_1,T_2]$.
\end{corollary}

The IM effect is an early-training phenomenon in which a separation in reconstruction error between inliers and outliers is guaranteed over a finite interval $[T_1, T_2]$, rather than after full convergence.
This provides a theoretical explanation for why under-fitted unsupervised models, such as autoencoders, are beneficial for UOD, as suggested in previous studies on the IM effect \citep{DBLP:conf/icml/KimHLKK24,zhang2024gradstop}.
\cref{thm_main} and \cref{cor:perfect_sep} yield several important implications that are closely aligned with practical scenarios. 

\begin{itemize}[topsep=0pt, leftmargin=1.5em, labelsep=0.5em]
    \item[$\diamond$] \textbf{(Persistence of the IM Effect)}  
    The persistence of the IM effect, characterized by the interval $[T_1, T_2]$, depends on both the \textit{data distribution} and \textit{weight initialization}. 
    From a data distribution perspective, this interval tends to be extended when inlier clusters are internally compact, i.e., when $\delta$ is large and $\varepsilon_c$ is small. 
    Specifically, this is reflected in the quantity $\delta - 4\varepsilon_c$, whose monotonic dependence is summarized in \cref{app:dependence}.
    In addition, the interval $[T_1, T_2]$ becomes longer when the cluster sizes are balanced, i.e., when the difference between $n_{\min}$ and $n_{\max}$ is small. 
    If the data distribution is imbalanced, the interval becomes shorter, as minor clusters may be misrecognized as outliers by the model. 
    Moreover, the interval increases when the outlier rate is small, since such outliers are harder to memorize during the early training stage, which is related to the empirical findings in \citep{DBLP:conf/aaai/ChoHBK25}.
    
    Finally, the interval is also affected by the weight initialization. 
    In particular, when the initial autoencoder captures inlier cluster structure more effectively, i.e., for smaller values of $ \widetilde{R}_{0} $, the IM effect persists over a longer training period. 
    This can be explained by the fact that such initialization biases the optimization trajectory toward inlier reconstruction, thereby delaying the model's tendency to fit outliers.
    
    \item[$\diamond$] \textbf{(Strength of the IM Effect)} 
    The strength of the IM effect depends on the \textit{data distribution}; specifically, the gap between inliers and outliers in terms of loss values is determined by their distance in the input space.
    The degree of separation in \cref{assumption:well_separation}, i.e., $\delta + 2\varepsilon_c$, reflects the \textit{proximity} among inlier clusters. 
    When $\delta + 2\varepsilon_c$ is small, the inliers are closely distributed, and in this case, outliers can be easily separated even if they are not far from the inliers. 
    On the other hand, when the inlier clusters are relatively far apart, outliers can be identified only if they are located sufficiently far from the inliers.
\end{itemize}

In conclusion, the persistence and strength of the IM effect are governed by multiple aspects of the data distribution, with the separation $\delta$ and the within-cluster variation $\varepsilon_c$ jointly influencing both.
Persistence is associated with a large value of $\delta - 4\varepsilon_c$, whereas strength is associated with a small value of $\delta + 2\varepsilon_c$.
Both quantities improve as $\varepsilon_c$ becomes smaller; 
a smaller $\varepsilon_c$ increases $\delta - 4\varepsilon_c$ and decreases $\delta + 2\varepsilon_c$, thereby extending the IM interval and strengthening the IM effect.

\begin{remark}[Milder Assumptions]
    An advantage of our theoretical framework is that, in contrast to existing works that analyze training dynamics of neural networks, it does not rely on strong overparameterization or restrictive architectural assumptions. 
    These analyses typically consider extremely wide networks and carefully structured output layers to guarantee perfect fitting \citep{DBLP:conf/icml/AroraDHLW19,DBLP:journals/corr/abs-1902-04674,DBLP:conf/aistats/LiSO20,DBLP:conf/iclr/LeeBNSPS18,DBLP:conf/nips/JacotHG18,DBLP:conf/nips/LeeXSBNSP19}.
    On the other hand, our analysis does not require perfect fitting, but only that inliers achieve smaller reconstruction error than outliers at some stage of training. 
    This enables milder assumptions, including no restriction on the number of hidden nodes and only a bounded spectral norm of the decoder matrix $A$, thereby narrowing the gap between theory and practice.
\end{remark}

\begin{remark}[Possible Extensions]
    For theoretical simplicity, we consider a setting where the decoder is fixed and only the encoder is trained.
    While this simplifies the analysis, practical autoencoders are typically trained jointly.
    Recent studies have analyzed the training dynamics in such jointly trained settings \citep{DBLP:journals/tit/NguyenWH21}.
    These results suggest that extending our analysis to jointly trained models may yield a more complete understanding of the IM effect.
    Moreover, our framework may extend to broader generative models, such as VAEs, where training dynamics have also been studied \citep{DBLP:journals/ml/WangH25}.
    We leave these directions for future work.
\end{remark}

\section{Experiments}
\label{sec:exp}


\subsection{Simulation Study for Validating Theoretical Results}
\label{sec:exp_simulation}

We conduct a simulation study to empirically validate the key implications of \cref{thm_main} and \cref{cor:perfect_sep}.
Specifically, we investigate the following four key main factors that affect the IM effect. 
The first three focus on the effects of data distribution, including 
(i) the within-cluster variation $\varepsilon_c$, 
(ii) cluster-size balance ($n_{\min} : n_{\max}$), and 
(iii) the outlier rate $\rho$. 
The fourth examines the effect of parameter initialization, namely 
(iv) the initialization quality $\widetilde{R}_0$.

\paragraph{Data Generation}
We generate $n = 2{,}000$ samples in $\mathbb{R}^2$ with $K = 3$ cluster centers. 
The centers $\mu_1, \mu_2,$ and $\mu_3$ are located in $[-10, 10]^2$, with pairwise separation $\delta = 16$.
Inlier samples are distributed equally across the three clusters ($n_{\min}=666$ and $n_{\max}=667$) unless otherwise noted.
Each inlier receives Gaussian perturbation $\xi_i \sim \mathcal{N}(0, \widetilde{\varepsilon}_c^{2} I),$ where $I$ denotes the identity matrix, and we can control the within-cluster variation $\varepsilon_c$ in \cref{assumption:clustered_model} using $\widetilde{\varepsilon}_c$.
For notational simplicity, we write $\varepsilon_c$ instead of $\widetilde{\varepsilon}_c$ whenever the context is clear in this section.
We set the default outlier rate as $\rho = 0.05$ ($100$ outliers), and the outliers are distributed uniformly in $[-18, 18]^2$ subject to a minimum distance of $10$ from every cluster center.
All features are min–max scaled to $[0, 1/\sqrt{2}]$ so that their norms are at most one.

\paragraph{Model Training and Evaluation Metrics}

We train a single-hidden-layer autoencoder in \cref{eq:autoencoder} with $\mathrm{tanh}$ activation 
and hidden size $H = 32$.
Both the encoder and decoder weights are trained and initialized using the Kaiming uniform scheme \cite{7410480}. 
We use the Adam optimizer \cite{kingma2014adam} with a learning rate of $2 \times 10^{-4}$ and a batch size of $500$ for $1{,}000$ epochs, corresponding to $4{,}000$ updates.

At each epoch $\tau$, we compute the per-sample reconstruction error $\|x_i - F_{W_\tau}(x_i)\|_2^2$ as the outlier score and evaluate the corresponding AUROC, denoted by $\textup{AUROC}(\tau)$. 
We define the \textit{IM window} as the maximal contiguous interval of epochs for which $\textup{AUROC}(\tau) \geq 0.9$, which serves as an empirical approximation of $[T_1, T_2]$.

\paragraph{Results}

We present the results showing how $\varepsilon_c$, $(n_{\mathrm{min}}, n_{\mathrm{max}})$, $\rho$, and $\widetilde{R}_0$ each affect the IM effect.

(i) \textit{Effect of within-cluster variation $\varepsilon_c$}:
    We vary $\varepsilon_c \in \{1.0, 2.5, 4.0\}$ while keeping all other terms fixed.
    \cref{fig:sim_exp1} shows the data distributions and AUROC trajectories of the trained autoencoder.
    As suggested by \cref{thm_main}, tighter clusters (smaller $\varepsilon_c$) yield a stronger and more sustained IM effect: the peak AUROC increases from $0.944$ ($\varepsilon_c = 4.0$) to $0.975$ ($\varepsilon_c = 1.0$), and the IM window widens from $193$ to $372$ epochs.
    This result confirms that smaller within-cluster variation makes the IM effect more stable and more pronounced.

    \begin{figure}[h!]
        \centering
        \includegraphics[width=0.9\linewidth]{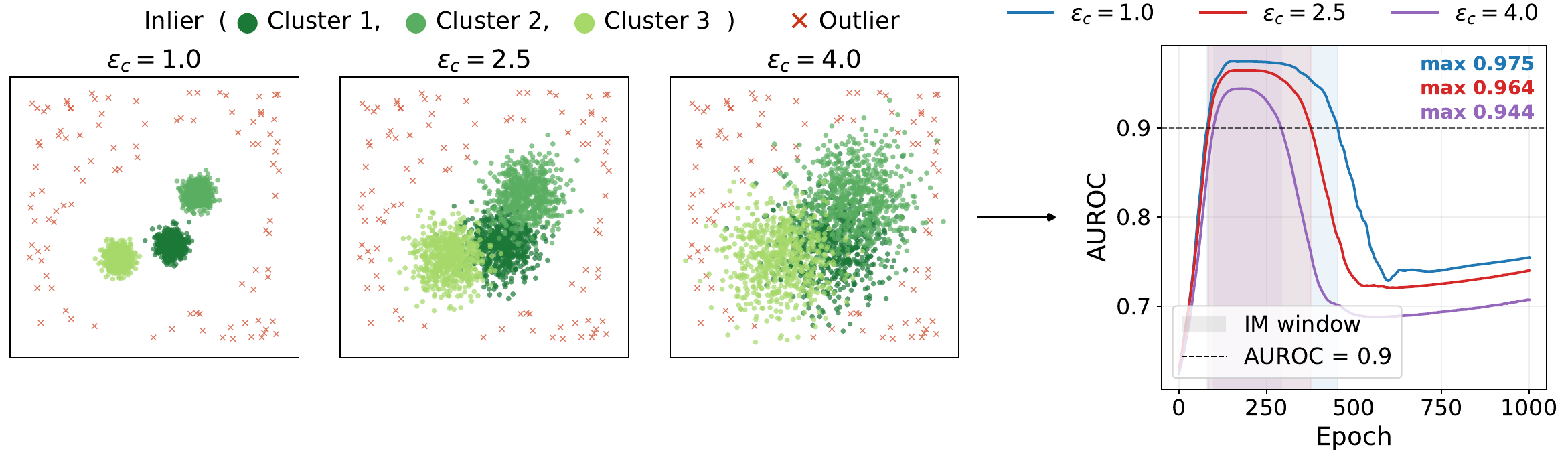}
        \vskip -0.1in
        \caption{
        Effect of within-cluster variation $\varepsilon_c$.
        (Left) Data distributions for $\varepsilon_c \in \{ 1.0, 2.5, 4.0 \}$.
        (Right) AUROC trajectories; shaded regions indicate the IM window.
        }
        \label{fig:sim_exp1}
        \vskip -0.1in
    \end{figure}

(ii) \textit{Effect of cluster-size balance}:
    With fixed $\varepsilon_c = 1.5$, we vary the cluster-size ratio $n_{\max}:n_{\min}$ across three levels: balanced ($1{:}1$), moderate ($11{:}1$), and heavy imbalance ($27{:}1$).
    The smallest cluster is assigned to a fixed spatial position across settings for fair comparison.
    \cref{fig:sim_exp2} shows that heavier imbalance degrades the persistence of the IM effect: the IM window narrows from $596$ epochs (balanced) to $462$ epochs (heavy imbalance).
    This is consistent with \cref{thm_main}, which suggests that imbalanced cluster sizes lead to a shorter duration of the IM effect.

    \begin{figure}[h!]
        \centering
        \includegraphics[width=0.9\linewidth]{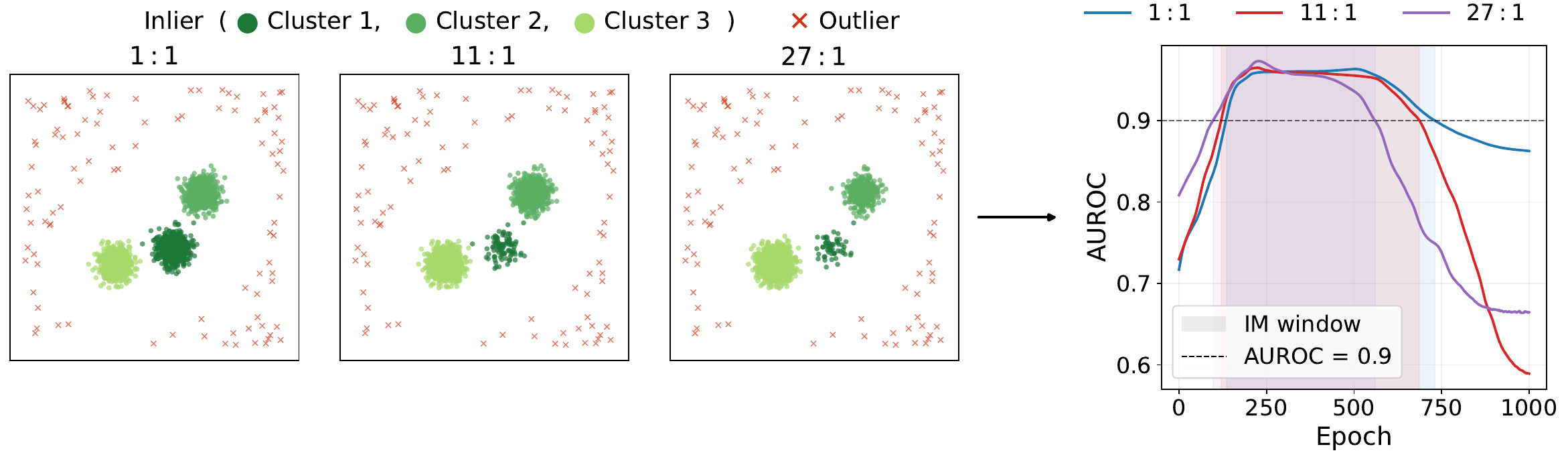}
        \vskip -0.1in
        \caption{
        Effect of cluster-size balance $n_{\max} : n_{\min}$.
        (Left) Data distributions for $(n_{\max} : n_{\min}) \in \{ (1:1), (11:1), (27:1) \}$.
        (Right) AUROC trajectories; shaded regions indicate the IM window.
        }
        \label{fig:sim_exp2}
        \vskip -0.1in
    \end{figure}

(iii) \textit{Effect of outlier rate $\rho$}:
    We fix $\varepsilon_c = 2.0$ and vary $\rho \in \{0.03,\, 0.10,\, 0.25\}$.
    \cref{fig:sim_exp4} shows that a higher $\rho$ substantially shortens the IM window from $391$ epochs ($\rho = 0.03$) to $94$ epochs ($\rho = 0.25$).
    This is consistent with our theoretical analysis, which suggests that a smaller $\rho$ leads to a prolonged IM effect.
    Interestingly, we also observe that strength of the IM effect improves as $\rho$ decreases, as the peak AUROC increases from $0.929$ to $0.968$.
    While this behavior is not directly explained by the current theory, it suggests that a lower $\rho$ may further facilitate the distinction between inliers and outliers in practice.

    \begin{figure}[h!]
        \centering
        \includegraphics[width=0.9\linewidth]{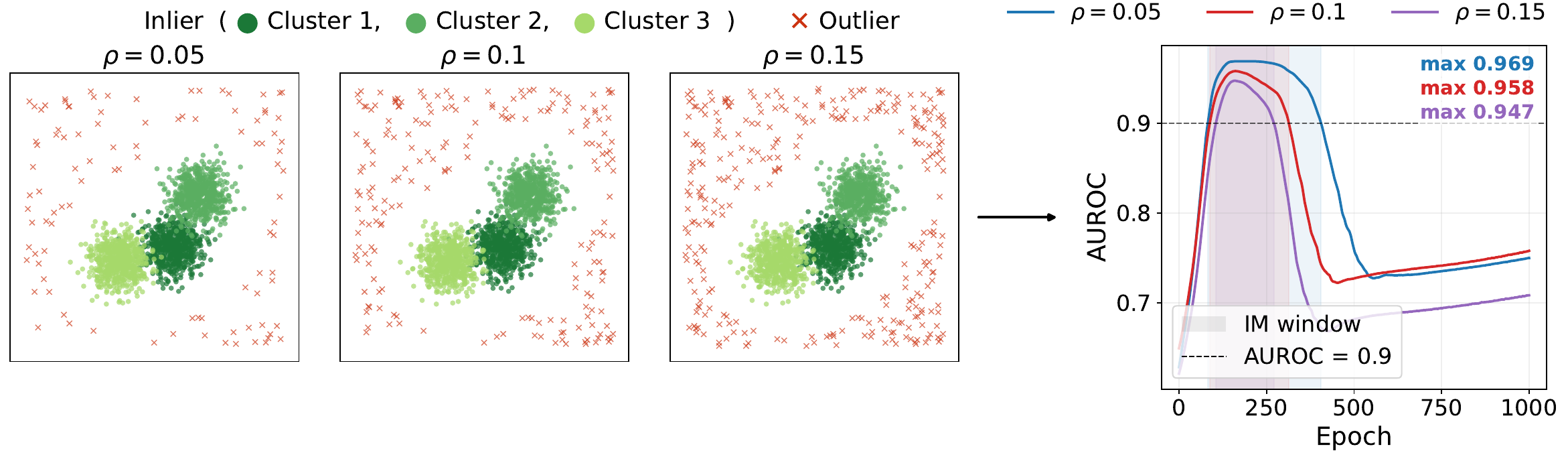}
        \vskip -0.1in
        \caption{
        Effect of outlier rate $\rho$.
        (Left) Data distributions for $\rho \in \{ 0.03, 0.10, 0.25 \}$.
        (Right) AUROC trajectories; shaded regions indicate the IM window.
        }
        \label{fig:sim_exp4}
    \end{figure}

(iv) \textit{Effect of initialization quality $\widetilde{R}_0$}:
    We fix $\varepsilon_c = 3.0$ and examine how the initial reconstruction quality affects the IM window.
    As discussed in \cref{sec:theory}, a low $\widetilde{R}_{0} = \|\mathcal{F}_{W_0}(\widetilde{X})-\widetilde{X}\|_2$ indicates a high quality of the initial parameter $W_0.$
    \textcolor{black}{\Cref{tab:sim_exp3} reports the IM window $[T_1,T_2]$ and the best AUROC for three initializations with different $\widetilde{R}_0$.}
    {\color{black}
    A smaller $\widetilde{R}_0$ produces a later $T_2$ and thus a wider IM window 
    at a similar best AUROC, supporting our theoretical finding that a smaller $\widetilde{R}_0$ lets the autoencoder capture the inlier cluster structure early, thereby prolonging the IM effect.
    Interestingly, although not predicted by the theory, $T_1$ also decreases as $\widetilde{R}_0$ becomes smaller, suggesting that better initialization accelerates the onset of the IM effect in practice.
    }
\begin{table}[h!]
    \vskip -0.15in
    \small
    \centering
    \caption{Effect of initialization error $\widetilde{R}_0$ on the IM window and detection performance (best AUROC).}
    \label{tab:sim_exp3}
    \vskip 0.1in
    \begin{tabular}{c|cccc}
        \toprule
        $\widetilde{R}_0$ & $T_1$ & $T_2$ & IM window length $(T_2-T_1)$ & Best AUROC \\
        \midrule
        0.40 & 37 & 509 & {472} & 0.961 \\
        0.55 & 56 & 433 & 377 & 0.962 \\
        0.60 & 94 & 456 & 362 & 0.964 \\
        \bottomrule
    \end{tabular}
    \vskip -0.15in
\end{table}

\subsection{Theory-Driven Guidelines for Real-World Data Analysis}
\label{sec:exp-real}

Our theoretical analysis in \cref{sec:main} explains the emergence of the IM effect and characterizes how its strength and persistence depend on data distribution and initialization.
Building on these insights, we propose two simple yet effective strategies for improving IM-based UOD solvers: (i) data preprocessing to shape the data distribution, and (ii) initialization schemes that guide optimization toward inlier-favorable dynamics.





\paragraph{Using Compact Representations}

While several key properties of the data distribution, such as the outlier rate $\rho$ and cluster sizes $(n_{\min}, n_{\max})$, are intrinsic and not directly controllable, data preprocessing can influence the within-cluster variation $\varepsilon_c$.
As suggested by our theoretical analysis, we again note that both the persistence and strength of the IM effect improve as $\varepsilon_c$ decreases.
In practice, this can be achieved by removing nuisance variations and noise that do not contribute to the underlying cluster structure of inliers.




A natural approach is to transform the data into a representation space where inlier clusters become more compact. 
One can leverage pre-trained foundation models used in the same domain, such as ViT \cite{DBLP:conf/iclr/DosovitskiyB0WZ21} for images, BERT \cite{DBLP:conf/naacl/DevlinCLT19} for text, and TabPFN \cite{DBLP:conf/iclr/Hollmann0EH23,DBLP:journals/corr/abs-2502-17361} for tabular data. 
These models, trained on large-scale real and synthetic datasets, capture meaningful structures while filtering out irrelevant variations, often yielding well-formed cluster structures.
When such models are unavailable, an alternative is to use self-supervised representations \citep{DBLP:journals/pami/GuiCZCSLT24}, such as contrastive learning \citep{DBLP:journals/corr/abs-1807-03748}.
Prior work shows that contrastive learning preserves cluster structure while reducing within-cluster variability \citep{DBLP:conf/iclr/0001YZJ23}.
As a result, these approaches reduce $\varepsilon_c$ and enhance both persistence and strength, improving IM-based outlier detection.
\paragraph{Initialization for Capturing Cluster Structure}

A smaller value of $\widetilde{R}_{0}$ leads to a wider IM window, extending the period during which inliers are reconstructed more accurately than outliers. 
This suggests that favorable initializations should capture the coarse structure of inliers while avoiding early fitting of sample-specific noise. 
In practice, however, obtaining such favorable initializations is challenging in the UOD setting, motivating the need for practical mechanisms to approximate them.

One practical approach is to refine the initialization by applying an exponential moving average (EMA, \cite{DBLP:conf/nips/TarvainenV17}) to the model parameters during the early training phase.
While EMA is commonly used to stabilize optimization, here it also plays an important role in suppressing the influence of outliers in the early stage.
By averaging parameters over successive updates, EMA biases the model toward consistently learned structures, which are dominated by inliers, and mitigates the effect of irregular gradient signals induced by outliers.
As a result, the resulting model better captures the underlying inlier structure and effectively reduces the initial reconstruction error $\widetilde{R}_0$.

\paragraph{\textcolor{black}{Empirical Validation on Real-World Data}}



We evaluate the two proposed strategies on the tabular and image datasets from \citep{han2022adbench} using two representative IM-based methods, ODIM \citep{DBLP:conf/icml/KimHLKK24} and ALTBI \citep{DBLP:conf/aaai/ChoHBK25}. 
For each method, we compare the vanilla version (using raw inputs and random initialization) with its variants incorporating representation learning and EMA warm-up. 
Specifically, we initialize the model parameters obtained by applying EMA with decay $0.999$ over the first $50$ training iterations. 
For representation learning, we use TabPFN \citep{DBLP:conf/iclr/Hollmann0EH23} and ViT \citep{DBLP:conf/iclr/DosovitskiyB0WZ21} embeddings for tabular and image data, respectively. 
For tabular data, TabPFN embeddings are applied only to datasets of moderate size due to scalability limitations. 
The UOD performance is measured by AUROC and AUPRC.
Further details are provided in \cref{sec-appen:data,sec-appen:implementation}.



\Cref{fig:real_auroc} shows that the proposed variants, which combine EMA warm-up and learned representations, outperform all baselines on the 47 tabular datasets and on \texttt{CIFAR10} and \texttt{FashionMNIST} (AUPRC results and ablation studies are given in \cref{sec-append:futehr_exp}). 
The only exception is \texttt{SVHN}, where the pre-trained ViT embeddings do not produce sufficiently compact inlier clusters.
Notably, the improvement is particularly pronounced on \texttt{CIFAR10}, where using pre-trained embeddings significantly outperforms using raw inputs.
This is expected, as background variations in raw images tend to dominate pairwise distances, whereas embeddings suppress such nuisance factors and better capture the underlying structure \cite{10.1109/TPAMI.2013.50,8503149}.


\begin{figure}[h]
    \vskip -0.1in
    \centering
    \includegraphics[width=0.85\linewidth]{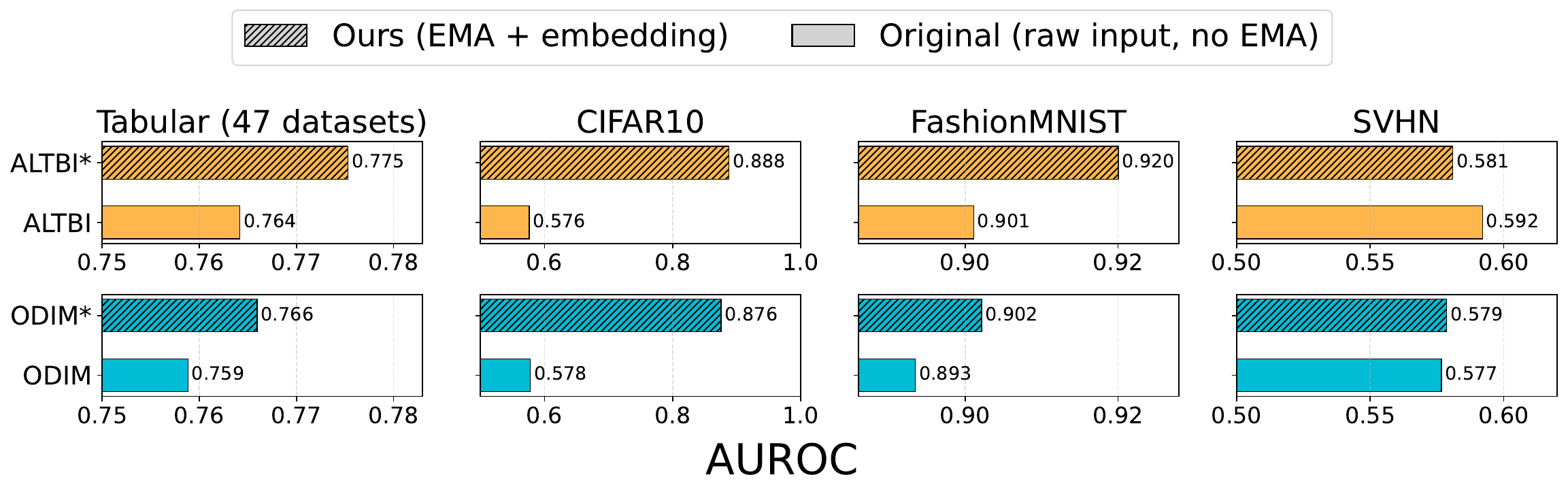}
    \caption{\textcolor{black}{
    Average over 47 tabular datasets (leftmost) and results on \texttt{CIFAR10}, \texttt{FashionMNIST}, and \texttt{SVHN}.
    Hatched bars (ALTBI* and ODIM*) are the proposed variants combining TabPFN/ViT embeddings with EMA warm-up, while the plain bars (ALTBI and ODIM) correspond to training on raw inputs without EMA.
   }
    }
    \label{fig:real_auroc}
    \vskip -0.1in
\end{figure}

Furthermore, we apply the proposed modifications to ALTBI and ODIM across all 57 datasets in \texttt{ADBench} to evaluate their overall effectiveness. 
We compare our methods against 22 existing baselines, including both classical machine-learning and deep-learning approaches. 
The results of the 22 baseline methods are directly taken from \citep{DBLP:journals/corr/abs-2305-18593}, while ALTBI and ODIM are evaluated using our own implementations. 
All reported results are averaged over three random initializations.
Note that the original ALTBI and ODIM already use embedding representations for image and text datasets; therefore, our modifications mainly consist of introducing the EMA warm-up and partially replacing tabular raw inputs with TabPFN embeddings.

The proposed guidelines improve the average AUROC of the original ALTBI from $0.757$ to $0.766$ and that of the original ODIM from $0.751$ to $0.757$. 
Moreover, among 22 existing baselines, the modified ALTBI achieves state-of-the-art performance with an average AUROC of $0.766$ and an average AUPRC of $0.352$. 
Detailed implementations and comprehensive experimental results, including per-dataset analyses and AUPRC comparisons, are provided in \cref{sec-append:futehr_exp}.  
Overall, these findings demonstrate the practical effectiveness and broad applicability of the proposed theoretical insights.

\section{Concluding Remarks}
\label{sec:conclusion}

We presented a theoretical study of the inlier-memorization (IM) effect, focusing on explaining why under-fitted autoencoders can reconstruct inliers earlier than outliers during training. 
Our analysis showed that the strength and persistence of the IM effect are governed by the data distribution and parameter initialization, providing a principled understanding of when IM-based outlier detection is effective. 
Simulation studies further verified the key implications of our theory under controlled settings.
Building on these insights, we proposed two simple yet practical strategies for real-world UOD: using compact representations and an EMA-based technique. 
Experiments on real datasets demonstrated that these strategies improve IM-based outlier detection performance, achieving the state-of-the-art results.

Future work includes extending the present analysis to deeper and jointly trained architectures, as well as to broader generative models such as VAEs, normalizing flows, and diffusion-based models. 
Another promising direction is to theoretically investigate more adaptive strategies for strengthening the IM effect, and to analyze how it can be integrated with active or semi-supervised frameworks for more label-efficient anomaly detection.



\begingroup
\small
\bibliographystyle{unsrt}
\bibliography{references}
\endgroup

\clearpage
\appendix

\crefalias{section}{appsec}
\crefalias{subsection}{appsubsec}
\crefalias{subsubsection}{appsubsubsec}

\numberwithin{equation}{section}
\numberwithin{figure}{section}
\numberwithin{table}{section}

\section{Theoretical Studies}\label{sec-appen:theory}

We present here the assumptions, formal statements, and proofs of \cref{thm_main} and \cref{cor:perfect_sep}, which establish the early inlier fitting behavior of the autoencoder. 
The proofs proceed in three steps. 

First, in \cref{sec:appen-rep_setup_theory}, we introduce the cluster-consistent subspace, study gradient descent on the cluster-centered inputs $\widetilde X$, and analyze the corresponding gradient descent dynamics.
Second, in \cref{sec:appen-transfer}, we show that the actual gradient descent trajectory on the observed inputs $X$ remains close to the cluster-centered dynamics, by establishing uniform bounds on the induced output and Jacobian discrepancies caused by $X-\widetilde X$.
Finally, in \cref{sec:appen-theory_mains}, we combine these ingredients to prove the occurrence of the IM effect first in the exact clustered case $(\varepsilon_c=0)$, and then in the general case $(\varepsilon_c>0)$. 
We note that our proof follows a similar structure to that of \cite{DBLP:conf/aistats/LiSO20}.


\subsection{Preliminaries and Technical Setup}

For a matrix $W=[w_1,\ldots,w_v]\in\mathbb R^{u\times v}$, we write $w:=\mathrm{vec}(W)\in\mathbb R^{uv}$ for the vector obtained by stacking the columns of $W$. 
We denote by $\mathrm{mat}(\cdot)$ the inverse operation of $\mathrm{vec}(\cdot)$, which reshapes a vector in $\mathbb R^{uv}$ back into a $u\times v$ matrix.

For given $W,$ we define the representative-input map and the observed-input map by
$$
\mathcal{F}_{W}(\widetilde{X}) := \begin{bmatrix}F_W(\widetilde x_1)\\ \vdots\\ F_W(\widetilde x_n)\end{bmatrix}\in\mathbb R^{np},
\quad
\mathcal{F}_{W}(X) := \begin{bmatrix}F_W(x_1)\\ \vdots\\ F_W(x_n)\end{bmatrix}\in\mathbb R^{np}.
$$
We also define
$$
X:=\begin{bmatrix}x_1\\ \vdots\\ x_n\end{bmatrix}\in\mathbb R^{np},
\quad
\widetilde X:=\begin{bmatrix}\widetilde x_1\\ \vdots\\ \widetilde x_n\end{bmatrix}\in\mathbb R^{np},
\quad
O:=\begin{bmatrix}o_1\\ \vdots\\ o_n\end{bmatrix}\in\mathbb R^{np},
\quad
\Xi:=\begin{bmatrix}\xi_1\\ \vdots\\ \xi_n\end{bmatrix}\in\mathbb R^{np}.
$$
Let
$$
E:=X-\widetilde X=\Xi+O.
$$
Note that under \cref{assumption:exact_clustered}, which leads to the idealized setting (no within-cluster variation) below, we have $E=O$.

For any $v\in\mathbb R^{np}$, we write
$$
v=\begin{bmatrix}v_1\\ \vdots\\ v_n\end{bmatrix},\quad v_i\in\mathbb R^p,
$$
and define
$
\|v\|_2:=\left(\sum_{i=1}^n\|v_i\|_2^2\right)^{1/2},
\|v\|_{2,\infty}:=\max_{1\le i\le n}\|v_i\|_2.
$
For matrices, let $\|\cdot\|$ denote the operator norm and $\|\cdot\|_F$ denote the Frobenius norm.

In addition to \cref{assumption:clustered_model} and \cref{assumption:outlier_rate}, we additionally impose the following assumptions for technical reasons. 
\cref{ass:bounded_sample_norm} states that the norm of each sample is uniformly bounded. 
We set the upper bound to one for mathematical convenience, although any positive constant can be used. 
\cref{assumption:outlier_magnitude} simply requires that the outlier perturbation is not excessively large, which is also a mild condition.

\begin{assumption}[Bounded Sample Norms]\label{ass:bounded_sample_norm}
    $ \|x_i\|_2 \le 1, \forall i\in [n]. $
\end{assumption}

\begin{assumption}[Bounded Outlier Magnitude]\label{assumption:outlier_magnitude}
    $\|o_i\|_2 \le \varepsilon_o$ for all $i \in \mathcal{O}.$
\end{assumption}

\subsection{Cluster-Centered Framework}
\label{sec:appen-rep_setup_theory}

\subsubsection{Idealized Cluster Setting and Subspace Projections}

For technical simplicity, we first consider the case $\varepsilon_c = 0$, so that $\xi_i = 0$ for all $i \in [n]$, and then extend the analysis to the general case $\varepsilon_c > 0$. 
The following assumption formalizes this idealized setting.

\begin{assumption}[No Within-Cluster Variation]\label{assumption:exact_clustered}
    $\xi_i=0$ for all $i=1,\dots,n$.
\end{assumption}


\begin{definition}[Support Subspace]\label{def:support_subspace}
Define
$$
S_+:=\left\{v=(v_1^\top,\ldots,v_n^\top)^\top\in\mathbb R^{np}:\exists u_1,\dots,u_K\in\mathbb R^p\text{ such that }v_i=u_k\text{ whenever }i\in\Lambda_k\right\}.
$$
We also define
$
S_-:=S_+^\perp,
$
and let $\Pi_{S_+}$ and $\Pi_{S_-}$ denote the orthogonal projections onto $S_+$ and $S_-$, respectively.
\end{definition}

\begin{remark}
$S_+$ consists of vectors whose components are identical within each cluster. 
That is, all entries corresponding to samples in the same cluster share the same vector value.
\end{remark}

\begin{lemma}[Explicit Expression for $\Pi_{S_+}$]\label{lem:projection_formula}
    For any $v=(v_1^\top,\dots,v_n^\top)^\top\in\mathbb R^{np}$ and any cluster $k\in[K]$, the orthogonal projection of $v$ onto $S_+$ is given by, 
    $$
    (\Pi_{S_+}v)_i=\frac{1}{n_k}\sum_{j\in\Lambda_k} v_j,
    \qquad \forall i\in\Lambda_k.
    $$
\end{lemma}

\begin{proof}
    Define $Pv\in\mathbb R^{np}$ by
    $
    (Pv)_i:=\frac{1}{n_k}\sum_{j\in\Lambda_k}v_j, \forall i \in \Lambda_k.
    $
    Clearly $Pv\in S_+$.
    It remains to show that $v-Pv\in S_+^\perp$.
    Fix $u\in S_+$.
    Then there exist $u_1,\dots,u_K\in\mathbb R^p$ such that $u_i=u_k$ whenever $i\in\Lambda_k$.
    Hence
    \begin{align*}
    \langle v-Pv,u\rangle
    &=
    \sum_{k=1}^K\sum_{i\in\Lambda_k}\left\langle v_i-\frac{1}{n_k}\sum_{j\in\Lambda_k}v_j,u_k\right\rangle
    =
    \sum_{k=1}^K\left\langle \sum_{i\in\Lambda_k}v_i-\sum_{i\in\Lambda_k}\frac{1}{n_k}\sum_{j\in\Lambda_k}v_j,u_k\right\rangle\\
    &=
    \sum_{k=1}^K\left\langle \sum_{i\in\Lambda_k}v_i-\sum_{j\in\Lambda_k}v_j,u_k\right\rangle
    =0.
    \end{align*}
    Therefore $v-Pv\in S_+^\perp$, so $Pv=\Pi_{S_+}v$.
\end{proof}


\begin{remark}[Interpretation of $\Pi_{S_+}(O)$]
By \cref{lem:projection_formula}, the operator $\Pi_{S_+}$ replaces, within each cluster, all block components corresponding to samples in the same cluster by their cluster-wise average. 
In particular, for the outlier perturbation matrix $O$, we have for each $i\in\Lambda_k$,
$$
(\Pi_{S_+}(O))_i=\frac{1}{n_k}\sum_{j\in\Lambda_k}o_j.
$$
Hence, $\Pi_{S_+}(O)$ represents the cluster-wise average of the outlier perturbations.
\end{remark}

\begin{lemma}[Uniform Bound for $\Pi_{S_+}(O)$]\label{lem:ae_projection}
    Under Assumptions~\ref{assumption:outlier_rate} and~\ref{assumption:outlier_magnitude}, the following inequality holds:
    $$
    \|\Pi_{S_+}(O)\|_{2,\infty}\le\rho \varepsilon_o .
    $$
\end{lemma}

\begin{proof}
    Let $P:=\Pi_{S_+}$.
    By \cref{lem:projection_formula}, for any $i\in\Lambda_k$,
    $
    (PO)_i=\frac{1}{n_k}\sum_{j\in\Lambda_k}o_j.
    $
    Hence
    \begin{align*}
    \|(PO)_i\|_2
    \le\frac{1}{n_k}\sum_{j\in\Lambda_k}\|o_j\|_2
    =\frac{1}{n_k}\sum_{j\in\mathcal{O}\cap\Lambda_k}\|o_j\|_2
    \le\frac{|\mathcal{O}\cap\Lambda_k|}{n_k} \varepsilon_o 
    \le\rho \varepsilon_o .
    \end{align*}
    Taking the maximum over all $i$ gives the result.
\end{proof}

\begin{lemma}[Uniform Bound for $\Pi_{S_+}(\Xi)$]\label{lem:xi_projection}
Under \cref{assumption:clustered_model}, the following inequality holds:
$$
\|\Pi_{S_+}(\Xi)\|_{2,\infty}\le \varepsilon_c .
$$
\end{lemma}

\begin{proof}
    Let $P:=\Pi_{S_+}$.
    By \cref{lem:projection_formula}, for any $i\in\Lambda_k$,
    $
    (P\Xi)_i=\frac{1}{n_k}\sum_{j\in\Lambda_k}\xi_j.
    $
    Hence
    $$
    \|(P\Xi)_i\|_2\le\frac{1}{n_k}\sum_{j\in\Lambda_k}\|\xi_j\|_2\le\frac{1}{n_k}\sum_{j\in\Lambda_k} \varepsilon_c = \varepsilon_c .
    $$
    Taking the maximum over $i$ completes the proof.
\end{proof}


\subsubsection{Jacobian Structure on the Cluster-Centered Inputs}

We define the reconstruction loss using cluster-centered inputs with observed targets by
$
\widetilde{\mathcal L}(W):=\frac{1}{2} \| \mathcal{F}_{W}(\widetilde{X}) - X \|_2^2
$
as well as the standard reconstruction loss by
$
\mathcal L(W):=\frac{1}{2}\|\mathcal{F}_{W}(X)-X\|_2^2.
$
The cluster-centered loss replaces each input sample by its corresponding cluster-centered input while preserving the original reconstruction targets.

For a given $W$, let $\widetilde J(W)$ and $J(W)$ denote the Jacobian matrices of the cluster-centered and observed-input maps, respectively:
$$
\widetilde J(W):=\frac{\partial \mathcal{F}_{W}(\widetilde{X})}{\partial w}\in\mathbb R^{np\times Hp}
\quad \text{ and } \quad
J(W):=\frac{\partial \mathcal{F}_{W}(X)}{\partial w}\in\mathbb R^{np\times Hp}.
$$
In addition, for two parameter matrices $W_1,W_2$, define their corresponding averaged Jacobians by
$$
\widetilde{J}(W_1,W_2):=\int_0^1 \widetilde{J} ( W_2+t(W_1-W_2))dt
$$
and
$$
J(W_1,W_2):=\int_0^1 J (W_2+t(W_1-W_2))dt.
$$
The following lemma shows that the averaged Jacobians yield an exact mean-value representation of the change in the autoencoder outputs between two parameter values.

\begin{lemma}[Mean-Value Jacobian Representation]\label{lem:avg_jacobian}
For any $W_1,W_2\in\mathbb R^{H\times p}$, then we have
$$
\mathcal{F}_{W_{1}}(\widetilde{X}) - \mathcal{F}_{W_{2}}(\widetilde{X}) = \widetilde{J}(W_1, W_2)(\mathrm{vec}(W_1)-\mathrm{vec}(W_2))
$$
and
$$
\mathcal{F}_{W_{1}}(X) - \mathcal{F}_{W_{2}}(X) = J(W_1, W_2)(\mathrm{vec}(W_1)-\mathrm{vec}(W_2))
$$
\end{lemma}

\begin{proof}
    Let $w_1:=\mathrm{vec}(W_1)$, $w_2:=\mathrm{vec}(W_2)$, and $\Delta w:=w_1-w_2$.
    Then, we define
    $ \gamma(t) := \mathcal{F}_{\mathrm{mat}(w_2+t\Delta w)}(\widetilde{X}), t\in[0,1].$
    By the chain rule,
    $ \gamma'(t) = \widetilde{J} (W_2+t(W_1-W_2))\Delta w. $
    Therefore
    \begin{align*}
        & \mathcal{F}_{W_{1}}(\widetilde{X}) - \mathcal{F}_{W_{2}}(\widetilde{X})
        =\gamma(1)-\gamma(0)
        =\int_0^1\gamma'(t)dt
        =\int_0^1 \widetilde{J} (W_2+t(W_1-W_2))\Delta wdt\\
        & = \left(\int_0^1 \widetilde{J} (W_2+t(W_1-W_2))dt\right)\Delta w
        = \widetilde{J} (W_1,W_2)(\mathrm{vec}(W_1)-\mathrm{vec}(W_2)).
    \end{align*}
    We can similarly prove the claim $\mathcal{F}_{W_{1}}(X) - \mathcal{F}_{W_{2}}(X) = J(W_1, W_2)(\mathrm{vec}(W_1)-\mathrm{vec}(W_2))$ using the same argument.
\end{proof}

We define the cluster-center map by
$$
\mathcal{F}_{W}(\mu) := \begin{bmatrix}F_W(\mu_1)\\ \vdots\\ F_W(\mu_K)\end{bmatrix}\in\mathbb R^{Kp},
$$
which stacks the autoencoder outputs evaluated at the cluster centers. 
We further define its Jacobian matrix with respect to the parameter vector $w$ by
$$
J_c(W):=\frac{\partial \mathcal{F}_{W}(\mu)}{\partial w}\in\mathbb R^{Kp\times Hp}.
$$
The following assumption imposes standard local regularity conditions on the Jacobian at the cluster centers, ensuring stable and well-conditioned optimization dynamics near the initialization.

\begin{assumption}[Cluster-Center Jacobian]\label{assumption:center_jacobian}
    Fix an initialization $W_0\in\mathbb R^{H\times p}$ and a radius $R_{\mathrm{loc}}>0$.
    Define
    $$
    \mathcal D(W_0,R_{\mathrm{loc}}):=\{W\in\mathbb R^{H\times p}:\|W-W_0\|_F\le R_{\mathrm{loc}}\}.
    $$
    Assume that there exist constants $\alpha_c,\beta_c,L_c>0$ such that, for every $W\in\mathcal D(W_0,R_{\mathrm{loc}})$,
    \begin{enumerate}[label=(\roman*)]
    \item for every unit vector $z\in\mathbb R^{Kp}$, $\alpha_c\le\|J_c(W)^\top z\|_2\le\beta_c$;
    \item for all $W_1,W_2\in\mathcal D(W_0,R_{\mathrm{loc}})$, $\|J_c(W_1)-J_c(W_2)\|\le L_c\|W_1-W_2\|_F$.
    \end{enumerate}
\end{assumption}

Under \cref{assumption:center_jacobian}, the cluster-representative Jacobian inherits the same structured regularity properties, summarized in the following proposition.


\begin{proposition}[Jacobian Structure of Cluster-Representative Inputs]\label{prop:clean_jacobian}
    Assume \cref{assumption:center_jacobian}.
    Fix
    $$
    \alpha:=\sqrt{n_{\min}}\alpha_c,
    \beta:=\sqrt{n_{\max}}\beta_c,
    L:=\sqrt{n_{\max}}L_c.
    $$
    Then, for every $W\in\mathcal D(W_0,R_{\mathrm{loc}})$,
    \begin{enumerate}[label=(\roman*)]
        \item $\mathrm{range}(\widetilde J(W))\subset S_+$;
        \item for every unit vector $u\in S_+$, $\alpha\le\|\widetilde J(W)^\top u\|_2\le\beta$;
        \item for all $W_1,W_2\in\mathcal D(W_0,R_{\mathrm{loc}})$, $\|\widetilde J(W_1)-\widetilde J(W_2)\|\le L\|W_1-W_2\|_F$.
    \end{enumerate}
\end{proposition}

\begin{proof}
    Define the duplication operator $D:\mathbb R^{Kp}\to\mathbb R^{np}$ by
    $ (Dz)_i:=z_k $ whenever $i\in\Lambda_k.$
    Then $D(\mathbb R^{Kp})=S_+, \mathcal{F}_{W}(\widetilde{X}) = D \mathcal{F}_{W}(\mu)W),$ and $\widetilde J(W)=DJ_c(W).$
    This proves $\mathrm{range}(\widetilde J(W))\subset S_+$.
    
    Now fix a unit vector $u\in S_+$.
    Since $u\in S_+$, there exist vectors $u_1,\dots,u_K\in\mathbb R^p$ such that $u_i=u_k$ for all $i\in\Lambda_k$.
    Then
    $
    \|u\|_2^2=\sum_{k=1}^Kn_k\|u_k\|_2^2=1
    $
    and
    $$
    D^\top u=\begin{bmatrix}n_1u_1\\ \vdots\\ n_Ku_K\end{bmatrix}.
    $$
    Hence, we get
    $
    \|D^\top u\|_2^2=\sum_{k=1}^Kn_k^2\|u_k\|_2^2.
    $
    Since $\sum_{k=1}^Kn_k\|u_k\|_2^2=1$, we also get
    $ n_{\min}\le\|D^\top u\|_2^2\le n_{\max} $
    and
    $ \widetilde J(W)^\top u=J_c(W)^\top D^\top u. $
    Let $z:=D^\top u/\|D^\top u\|_2$.
    Then $\|z\|_2=1$, so by \cref{assumption:center_jacobian},
    $$
    \alpha_c\|D^\top u\|_2\le\|\widetilde J(W)^\top u\|_2\le\beta_c\|D^\top u\|_2.
    $$
    Using the bound on $\|D^\top u\|_2$ yields
    $
    \sqrt{n_{\min}}\alpha_c\le\|\widetilde J(W)^\top u\|_2\le\sqrt{n_{\max}}\beta_c.
    $    
    Finally, we have
    $$
    \widetilde J(W_1)-\widetilde J(W_2)=D(J_c(W_1)-J_c(W_2)).
    $$
    Since $\|D\|=\sqrt{n_{\max}}$, we obtain
    $$
    \|\widetilde J(W_1)-\widetilde J(W_2)\|\le\sqrt{n_{\max}}\|J_c(W_1)-J_c(W_2)\|\le\sqrt{n_{\max}}L_c\|W_1-W_2\|_F.
    $$
    This completes the proof.
\end{proof}

\subsubsection{Gradient Descent Dynamics on Cluster-Centered Inputs}

Let $(\widetilde W_\tau)_{\tau\ge0}$ denote the matrix-form gradient descent sequence for the cluster-centered loss $\widetilde{\mathcal L}(W)$, initialized at $\widetilde W_0=W_0$, and generated by
$$
\widetilde w_{\tau+1}
=
\widetilde w_\tau-\eta \widetilde J(\widetilde W_\tau)^\top \widetilde r_\tau,
\quad \text{where} \quad
\widetilde r_\tau:=\mathcal{F}_{\widetilde W_\tau}(\widetilde X)-X,
$$
and 
$
\widetilde W_\tau:=\mathrm{mat}(\widetilde w_\tau).
$
Let
$
 R_{0} :=\|\widetilde r_0\|_2,
 {\color{black}{\widehat R_0:=\widetilde{R}_{0}+\sqrt n \varepsilon_c+\sqrt{\rho n} \varepsilon_o}} 
$
and for any $\nu>0$, we also define the iteration point
$$
T_\star(\nu):=\left\lceil\frac{2}{\eta\alpha^2}\log_+\left(\frac{ {\color{black} \widehat R_0} }{\nu}\right)\right\rceil.
$$
To ensure that the gradient descent trajectory remains within the local neighborhood where the Jacobian regularity conditions hold, we impose the following radius condition.

\begin{assumption}[Radius condition]\label{assumption:radius_condition}
    $ R_{\mathrm{loc}}\ge\frac{2\beta}{\alpha^2} {\color{black} \widehat R_0} . $
\end{assumption}

\begin{remark}
{\color{black}
    Recall that
    $ \widetilde{R}_{0} = \|\mathcal F_{W_0}(\widetilde X)-\widetilde X\|_2. $
    Since $X=\widetilde X+E=\widetilde X+\Xi+O,$ we have the following relation between $R_0$ and $\widehat R_0$:
    $$ R_0 \le \widetilde{R}_{0}+\|\Xi\|_2+\|O\|_2 \le \widetilde{R}_{0}+\sqrt n \varepsilon_c+\sqrt{\rho n} \varepsilon_o=\widehat R_0.$$
}
\end{remark}

The following lemma summarizes the key gradient descent dynamics under the cluster-centered loss.
\begin{lemma}[Gradient Descent Dynamics on Cluster Representatives]
\label{lem:clean_meta}
    Assume \cref{assumption:center_jacobian} and \cref{assumption:radius_condition}.
    If
    $
    0<\eta\le\min\left\{\frac{\alpha^2}{8\beta^4},\frac{\alpha^2}{2L\beta^2 {\color{black} \widehat R_0} }\right\},
    $
    then we have for all $\tau\ge0$ that
    $
    \Pi_{S_-}(\widetilde r_\tau)=\Pi_{S_-}(\widetilde r_0),
    $
    $
    \|\Pi_{S_+}(\widetilde r_\tau)\|_2^2\le(1-\eta\alpha^2)^\tau\|\Pi_{S_+}(\widetilde r_0)\|_2^2,
    $
    and
    $
    \|\widetilde W_\tau-W_0\|_F\le\frac{2\beta}{\alpha^2} {\color{black} \widehat R_0} .
    $
    Moreover, if $\|\Pi_{S_+}(E)\|_{2,\infty}\le\nu$, then every $\tau\ge T_\star(\nu)$ satisfies
    $$
    \| \mathcal{F}_{\widetilde{W}_{\tau}}(\widetilde{X}) - \widetilde X\|_{2,\infty}\le2\nu.
    $$
\end{lemma}

\begin{proof}
    Write
    $
    \overline r_\tau:=\Pi_{S_+}(\widetilde r_\tau),
    \overline e_\tau:=\Pi_{S_-}(\widetilde r_\tau),
    \widetilde r_\tau=\overline r_\tau+\overline e_\tau.
    $
    By \cref{lem:avg_jacobian} and the gradient descent update, we have
    $
    \widetilde r_{\tau+1}=\widetilde r_\tau-\eta C_\tau\widetilde r_\tau,
    $
    where
    $
    C_\tau:=\widetilde J(\widetilde W_{\tau+1},\widetilde W_\tau)\widetilde J(\widetilde W_\tau)^\top.
    $
    
    Since $\mathrm{range}(\widetilde J(\widetilde W_\tau))\subset S_+$ by \cref{prop:clean_jacobian}(i),
    we have
    $ \widetilde J(\widetilde W_\tau)^\top v=0\quad $ for all $ v\in S_-.$
    Hence
    $
    \widetilde J(\widetilde W_\tau)^\top\overline e_\tau=0,
    $
    and 
    $
    C_\tau\overline e_\tau=0.
    $
    Therefore
    $
    \Pi_{S_-}(\widetilde r_{\tau+1})=\Pi_{S_-}(\widetilde r_\tau-\eta C_\tau\widetilde r_\tau)=\Pi_{S_-}(\widetilde r_\tau).
    $
    By induction,
    $$
    \Pi_{S_-}(\widetilde r_\tau)=\Pi_{S_-}(\widetilde r_0)\quad\text{for all }\tau\ge0.
    $$
    
    Moreover, since $C_\tau\overline e_\tau=0$, we have
    $ \overline r_{\tau+1}=\overline r_\tau-\eta C_\tau\overline r_\tau, $
    which results in
    \begin{equation}\label{eq:clean_rplus}
    \|\overline r_{\tau+1}\|_2^2=\|\overline r_\tau\|_2^2-2\eta\langle\overline r_\tau,C_\tau\overline r_\tau\rangle+\eta^2\|C_\tau\overline r_\tau\|_2^2.
    \end{equation}
    
    We also show that $\|\overline r_\tau\|_2\le {\color{black} \widehat R_0} $ for all $\tau\ge0$.
    For $\tau=0$, it is clear since $\|\overline r_\tau\|_2 \le R_0\le {\color{black} \widehat R_0}$.
    Assume it holds at time $\tau$.
    Then
    $$
    \|\widetilde W_{\tau+1}-\widetilde W_\tau\|_F=\eta\|\widetilde J(\widetilde W_\tau)^\top\widetilde r_\tau\|_2=\eta\|\widetilde J(\widetilde W_\tau)^\top\overline r_\tau\|_2\le\eta\beta\|\overline r_\tau\|_2\le\eta\beta {\color{black} \widehat R_0} .
    $$
    By \cref{prop:clean_jacobian}(iii),
    $
    \|\widetilde J(\widetilde W_{\tau+1},\widetilde W_\tau)-\widetilde J(\widetilde W_\tau)\|\le\frac{L}{2}\|\widetilde W_{\tau+1}-\widetilde W_\tau\|_F\le\frac{L\eta\beta}{2} {\color{black} \widehat R_0} \le\frac{\alpha^2}{4\beta}.
    $
    
    Now fix a unit vector $u\in S_+$.
    Then
    $
    \langle u,C_\tau u\rangle=\left\langle \widetilde J(\widetilde W_{\tau+1},\widetilde W_\tau)^\top u,\widetilde J(\widetilde W_\tau)^\top u\right\rangle.
    $
    Using \cref{prop:clean_jacobian}(ii),
    \begin{align*}
        \langle u,C_\tau u\rangle
        &=
        \|\widetilde J(\widetilde W_\tau)^\top u\|_2^2+\left\langle (\widetilde J(\widetilde W_{\tau+1},\widetilde W_\tau)-\widetilde J(\widetilde W_\tau))^\top u,\widetilde J(\widetilde W_\tau)^\top u\right\rangle\\
        &\ge
        \alpha^2-\|\widetilde J(\widetilde W_{\tau+1},\widetilde W_\tau)-\widetilde J(\widetilde W_\tau)\|\|\widetilde J(\widetilde W_\tau)^\top u\|_2
        \ge
        \alpha^2-\frac{\alpha^2}{4}
        =
        \frac{3\alpha^2}{4},
    \end{align*}
    and
    \begin{align*}
        \|C_\tau u\|_2
        \le
        \|\widetilde J(\widetilde W_{\tau+1},\widetilde W_\tau)\|\|\widetilde J(\widetilde W_\tau)^\top u\|_2
        \le
        \left(\beta+\frac{\alpha^2}{4\beta}\right)\beta
        \le
        2\beta^2.
    \end{align*}
    
    Applying these bounds to \cref{eq:clean_rplus},
    $
    \|\overline r_{\tau+1}\|_2^2\le\left(1-\frac{3}{2}\eta\alpha^2+4\eta^2\beta^4\right)\|\overline r_\tau\|_2^2.
    $
    Since $\eta\le\alpha^2/(8\beta^4)$,
    $
    4\eta^2\beta^4\le\frac{1}{2}\eta\alpha^2,
    $
    and therefore
    $
    \|\overline r_{\tau+1}\|_2^2\le(1-\eta\alpha^2)\|\overline r_\tau\|_2^2.
    $
    In particular, we have that
    $
    \|\overline r_{\tau+1}\|_2\le\|\overline r_\tau\|_2\le {\color{black} \widehat R_0} ,
    $
    which leads to
    $
    \|\Pi_{S_+}(\widetilde r_\tau)\|_2^2\le(1-\eta\alpha^2)^\tau\|\Pi_{S_+}(\widetilde r_0)\|_2^2.
    $
    
    Using the gradient descent update, it holds that
    $
    \|\widetilde W_{\tau+1}-\widetilde W_\tau\|_F\le\eta\beta\|\overline r_\tau\|_2\le\eta\beta(1-\eta\alpha^2)^{\tau/2} {\color{black} \widehat R_0} .
    $
    Therefore, we get
    \begin{align*}
        \|\widetilde W_\tau-W_0\|_F
        \le
        \sum_{s=0}^{\tau-1}\|\widetilde W_{s+1}-\widetilde W_s\|_F
        \le
        \eta\beta R_{0} \sum_{s=0}^{\infty}(1-\eta\alpha^2)^{s/2}
        \le
        \frac{2\beta}{\alpha^2} R_{0} 
        \le
        R_{\mathrm{loc}}.
    \end{align*}
    
    Finally, since the inputs $\widetilde x_i$ are constant inside each cluster, we have $ \mathcal{F}_{W}(\widetilde{X}) \in S_+$ for every $W$.
    Also $\widetilde X\in S_+$.
    Hence
    $
    \mathcal{F}_{\widetilde{W}_{\tau}}(\widetilde{X}) - \widetilde X\in S_+.
    $
    Since
    $
    \widetilde r_\tau = \mathcal{F}_{\widetilde{W}_{\tau}}(\widetilde{X}) - X
    = (\mathcal{F}_{\widetilde{W}_{\tau}}(\widetilde{X}) - \widetilde X) - E,
    $
    we get
    $
    \mathcal{F}_{\widetilde{W}_{\tau}}(\widetilde{X}) - \widetilde X=\Pi_{S_+}(\widetilde r_\tau)+\Pi_{S_+}(E).
    $
    Therefore
    $
    \| \mathcal{F}_{\widetilde{W}_{\tau}}(\widetilde{X}) -\widetilde X\|_{2,\infty}\le\|\Pi_{S_+}(\widetilde r_\tau)\|_2+\|\Pi_{S_+}(E)\|_{2,\infty}.
    $
    If $\|\Pi_{S_+}(E)\|_{2,\infty}\le\nu$ and $\tau\ge T_\star(\nu)$, then
    $
    \|\Pi_{S_+}(\widetilde r_\tau)\|_2\le(1-\eta\alpha^2)^{\tau/2} {\color{black} \widehat R_0} \le\nu.
    $
    Hence, we conclude that
    $
    \| \mathcal{F}_{\widetilde{W}_{\tau}}(\widetilde{X}) - \widetilde X\|_{2,\infty}\le2\nu.
    $
\end{proof}

\subsection{From Cluster-Centered Inputs to Observed Inputs}
\label{sec:appen-transfer}

\subsubsection{Output and Jacobian Perturbation Bounds}\label{sec:appen-outjbounds}

Before deriving the discrepancy bounds, we impose the following mild regularity condition on the activation function. 
This assumption is mild and is satisfied by many standard activations. 
For example, sigmoid satisfies $(B_1,B_2)=(1/4,1/(6\sqrt3))$, while ReLU satisfies $(B_1,B_2)=(1,0)$ almost everywhere, i.e., for all $t\neq 0$.

\begin{assumption}[Activation bounds]\label{assumption:phi_bounds}
    Assume that $|\phi'(t)|\le B_1$ and $|\phi''(t)|\le B_2$ for all $t\in\mathbb R$.
\end{assumption}

Let
$$
R_\mu:=\max_{1\le k\le K}\|\mu_k\|_2,
\quad
R_W:=\|W_0\|_F+R_{\mathrm{loc}},
$$
and define
$$
\Delta_f:=\|A\|B_1R_W(\varepsilon_c+\varepsilon_o),
\quad
\Delta_J:=\sqrt n\,\|A\|\bigl(B_1+B_2R_WR_\mu\bigr)(\varepsilon_c+\varepsilon_o).
$$


The following proposition quantifies how the mismatch between $X$ and $\widetilde X$ affects the autoencoder outputs and Jacobians.

\begin{proposition}[Output and Jacobian Discrepancy Bounds]\label{prop:input_perturb}
    Under \cref{assumption:phi_bounds}, we have for every $W\in\mathcal D(W_0,R_{\mathrm{loc}})$ that
    $
    \|\mathcal{F}_{W}(X) - \mathcal{F}_{W}(\widetilde{X}) \|_{2,\infty}\le\Delta_f
    $
    and
    $
    \|J(W)-\widetilde J(W)\|\le\Delta_J.
    $
\end{proposition}

\begin{proof}
    Fix $W\in\mathcal D(W_0,R_{\mathrm{loc}})$.
    Then $\|W\|\le\|W\|_F\le R_W$.

    For each $i$, since
    $
    x_i-\widetilde x_i=\xi_i+o_i,
    $
    Assumptions~\ref{assumption:clustered_model} and~\ref{assumption:outlier_magnitude} give
    $
    \|x_i-\widetilde x_i\|_2
    \le
    \|\xi_i\|_2+\|o_i\|_2
    \le
    \varepsilon_c+\varepsilon_o.
    $
    For the output bound, for each $i$,
    \begin{align*}
    \|F_W(x_i)-F_W(\widetilde x_i)\|_2
    &=
    \|A(\phi(Wx_i)-\phi(W\widetilde x_i))\|_2\\
    &\le
    \|A\|\|\phi(Wx_i)-\phi(W\widetilde x_i)\|_2\\
    &\le
    \|A\|B_1\|W(x_i-\widetilde x_i)\|_2\\
    &\le
    \|A\|B_1R_W\|x_i-\widetilde x_i\|_2\\
    &\le
    \|A\|B_1R_W(\varepsilon_c+\varepsilon_o)
    =
    \Delta_f.
    \end{align*}
    Taking the maximum over $i$ gives
    $$
    \|\mathcal{F}_{W}(X)-\mathcal{F}_{W}(\widetilde X)\|_{2,\infty}\le\Delta_f.
    $$

    For the Jacobian bound, fix a matrix $U\in\mathbb R^{H\times p}$ with $\|U\|_F=1$.
    For each $i$, let
    $$
    D_i:=\mathrm{diag}(\phi'(Wx_i)),
    \qquad
    \widetilde D_i:=\mathrm{diag}(\phi'(W\widetilde x_i)).
    $$
    The directional derivative of $F_W(x)$ in direction $U$ is
    $A\mathrm{diag}(\phi'(Wx))Ux$.
    Hence
    \begin{align*}
        (J(W)-\widetilde J(W))[U]_i
        &=
        AD_iU(x_i-\widetilde x_i)+A(D_i-\widetilde D_i)U\widetilde x_i.
    \end{align*}
    Therefore
    \begin{align*}
        \|(J(W)-\widetilde J(W))[U]_i\|_2
        &\le
        \|A\|\|D_i\|\|U(x_i-\widetilde x_i)\|_2
        +\|A\|\|D_i-\widetilde D_i\|\|U\widetilde x_i\|_2\\
        &\le
        \|A\|B_1\|x_i-\widetilde x_i\|_2
        +\|A\|B_2\|W(x_i-\widetilde x_i)\|_\infty\|\widetilde x_i\|_2\\
        &\le
        \|A\|\bigl(B_1+B_2R_WR_\mu\bigr)\|x_i-\widetilde x_i\|_2\\
        &\le
        \|A\|\bigl(B_1+B_2R_WR_\mu\bigr)(\varepsilon_c+\varepsilon_o).
    \end{align*}
    Summing over $i$ and using $\|U\|_F=1$, we obtain
    $$
    \|(J(W)-\widetilde J(W))[U]\|_2
    \le
    \sqrt n\,\|A\|\bigl(B_1+B_2R_WR_\mu\bigr)(\varepsilon_c+\varepsilon_o)
    =
    \Delta_J.
    $$
    Taking the supremum over $\|U\|_F=1$ concludes that
    $$
    \|J(W)-\widetilde J(W)\|\le\Delta_J.
    $$
\end{proof}

\subsubsection{Deviation from the Cluster-Centered Path}

Let $(W_\tau)_{\tau\ge 0}$ be the gradient descent sequence for the standard autoencoder loss:
$$
w_{\tau+1}=w_\tau-\eta J(W_\tau)^\top r_\tau,
\quad \text{where} \quad
r_\tau := \mathcal{F}_{W_\tau}(X) - X,
$$
and $W_\tau:=\mathrm{mat}(w_\tau).$
Let
$$
\beta_X:=\beta+\Delta_J,
\quad
\kappa:=L {\color{black} \widehat R_0} +\beta\beta_X,
\quad
\gamma := \Delta_J  {\color{black} \widehat R_0}  + \beta_X \sqrt{n} \Delta_f.
$$
For any integer $T\ge0$, we further introduce
$$
\Delta_{\mathrm{path}}(T):=\eta\gamma T\exp(\eta\kappa T),
\quad
\zeta_T:=\Delta_f+\beta\Delta_{\mathrm{path}}(T).
$$
Using \cref{prop:clean_jacobian,lem:clean_meta,prop:input_perturb}, we obtain the following proposition, which shows that the actual gradient descent path remains close to the cluster-centered path, together with a corresponding bound on the output discrepancy.


\begin{proposition}[Perturbation of gradient descent path]\label{prop:path_perturb}
    Under \cref{assumption:phi_bounds},
    if
    $
    \Delta_{\mathrm{path}}(T)\le R_{\mathrm{loc}}-\frac{2\beta}{\alpha^2} {\color{black} \widehat R_0} ,
    $
    then for every $\tau\le T$, we have
    $$
    \|W_\tau-\widetilde W_\tau\|_F\le\Delta_{\mathrm{path}}(T)
    $$
    and
    $$
    \| \mathcal{F}_{W_\tau}(X) - \mathcal{F}_{\widetilde{W}_{\tau}}(\widetilde{X}) \|_{2,\infty}\le\zeta_T.
    $$
\end{proposition}

\begin{proof}
    Let
    $$
    d_\tau:=\|W_\tau-\widetilde W_\tau\|_F,
    \quad
    p_\tau:=\| \mathcal{F}_{W_\tau}(X) - \mathcal{F}_{\widetilde{W}_{\tau}}(\widetilde{X}) \|_2.
    $$
    We prove by induction that $W_\tau\in\mathcal D(W_0,R_{\mathrm{loc}})$ and $d_\tau\le\Delta_{\mathrm{path}}(T)$ for all $\tau\le T$.
    
    At $\tau=0$, we have $d_0=0$ and $W_0=\widetilde W_0$.
    
    Assume the claim holds up to time $\tau\le T-1$.
    Then $W_\tau,\widetilde W_\tau\in\mathcal D(W_0,R_{\mathrm{loc}})$.
    Using the update rules, we can derive
    $
    d_{\tau+1}
    =
    \|W_{\tau+1}-\widetilde W_{\tau+1}\|_F
    \le
    d_\tau+\eta\|(J(W_\tau)-\widetilde J(\widetilde W_\tau))^\top\widetilde r_\tau\|_2+\eta\|J(W_\tau)^\top(r_\tau-\widetilde r_\tau)\|_2.
    $
    By \cref{prop:input_perturb} and \cref{prop:clean_jacobian}(iii),
    $$
    \|J(W_\tau)-\widetilde J(\widetilde W_\tau)\|\le\|J(W_\tau)-\widetilde J(W_\tau)\|+\|\widetilde J(W_\tau)-\widetilde J(\widetilde W_\tau)\|\le\Delta_J+Ld_\tau.
    $$
    Also
    $$
    \|J(W_\tau)\|\le\|J(W_\tau)-\widetilde J(W_\tau)\|+\|\widetilde J(W_\tau)\|\le\Delta_J+\beta=\beta_X.
    $$
    Furthermore, by \cref{lem:clean_meta}, we know that
    $
    \|\widetilde r_\tau\|_2\le {\color{black} \widehat R_0} ,
    $
    leading to
    $
    d_{\tau+1}\le d_\tau+\eta(\Delta_J+Ld_\tau) {\color{black} \widehat R_0} +\eta\beta_Xp_\tau.
    $    
    Now, we get
    $
    p_\tau
    =
    \|\mathcal{F}_{W_\tau}(X) - \mathcal{F}_{\widetilde{W}_{\tau}}(\widetilde{X}) \|_2
    \le
    \|\mathcal{F}_{W_\tau}(X) - \mathcal{F}_{W_{\tau}}(\widetilde{X}) \|_2 + \| \mathcal{F}_{W_{\tau}}(\widetilde{X}) - \mathcal{F}_{\widetilde{W}_{\tau}}(\widetilde{X}) \|_2
    \le
    \sqrt{n} \Delta_f + \|\widetilde J(W_\tau,\widetilde W_\tau)\|\|W_\tau-\widetilde W_\tau\|_F
    \le
    \sqrt{n} \Delta_f + \beta d_\tau.
    $
    Substituting this into the recursion for $d_{\tau+1}$ gives
    $
    d_{\tau+1}\le(1+\eta\kappa)d_\tau+\eta\gamma.
    $
    Since $d_0=0$, the induction gives
    $$
    d_\tau\le\eta\gamma\tau(1+\eta\kappa)^\tau\le\eta\gamma\tau e^{\eta\kappa\tau}\le\Delta_{\mathrm{path}}(T)
    $$
    for all $\tau\le T$.
    
    Also, by \cref{lem:clean_meta}, it holds that
    $
    \|\widetilde W_\tau-W_0\|_F\le\frac{2\beta}{\alpha^2} {\color{black} \widehat R_0} .
    $
    Therefore, we get
    $$
    \|W_\tau-W_0\|_F\le\|W_\tau-\widetilde W_\tau\|_F+\|\widetilde W_\tau-W_0\|_F\le\Delta_{\mathrm{path}}(T)+\frac{2\beta}{\alpha^2} {\color{black} \widehat R_0} \le R_{\mathrm{loc}}.
    $$
    So $W_\tau\in\mathcal D(W_0,R_{\mathrm{loc}})$ for all $\tau\le T$.
    
    Finally, using the bound on $p_\tau$ and the definition of $\zeta_T$, we can conclude that
    \begin{align*}
    \| \mathcal{F}_{W_\tau}(X) - \mathcal{F}_{\widetilde{W}_{\tau}}(\widetilde{X}) \|_{2,\infty}
    &\le
    \| \mathcal{F}_{W_\tau}(X) - \mathcal{F}_{W_{\tau}}(\widetilde{X}) \|_{2,\infty}
    + \| \mathcal{F}_{W_{\tau}}(\widetilde{X}) - \mathcal{F}_{\widetilde{W}_{\tau}}(\widetilde{X}) \|_2\\
    &\le
    \Delta_f+\beta d_\tau
    \le
    \Delta_f+\beta\Delta_{\mathrm{path}}(T)
    =
    \zeta_T.
    \end{align*}
\end{proof}

\subsection{Theoretical Results for the IM Effect}
\label{sec:appen-theory_mains}

\subsubsection{IM Effect in the Exact Clustered Case $( \varepsilon_c  = 0)$}
\label{app_secA_cluster_case}

We first show that, in the exact clustered case $(\varepsilon_c=0)$, the autoencoder reconstructs the underlying cluster structure before fitting sample-specific perturbations.

\begin{theorem}[Early Memorization of Cluster Information in the Exact Clustered Case]
\label{thm:ae_exact}
    Suppose that Assumptions \ref{assumption:outlier_rate}, \ref{assumption:outlier_magnitude}, \ref{assumption:exact_clustered}, \ref{assumption:center_jacobian}, \ref{assumption:radius_condition}, and \ref{assumption:phi_bounds} hold. 
    Let $\alpha,\beta,L$ be as defined in \cref{prop:clean_jacobian}, and assume that
    $
    0<\eta\le\min\left\{\frac{\alpha^2}{8\beta^4},\frac{\alpha^2}{2L\beta^2 {\color{black} \widehat R_0} }\right\}.
    $
    Fix an integer $T\ge T_\star(\rho \varepsilon_o )$.
    If
    $
    \Delta_{\mathrm{path}}(T)\le R_{\mathrm{loc}}-\frac{2\beta}{\alpha^2} {\color{black} \widehat R_0} 
    $
    and
    $
    \delta:=\min_{a\neq b}\|\mu_a-\mu_b\|_2>2\zeta_T+4\rho \varepsilon_o ,
    $
    then
    $$
    \|\mathcal{F}_{W_T}(X)-\widetilde X\|_{2,\infty}\le\zeta_T+2\rho \varepsilon_o .
    $$
    Equivalently,
    $$
    \|F_{W_T}(x_i)-\mu_{g(i)}\|_2\le\zeta_T+2\rho \varepsilon_o \quad\text{for all }i.
    $$
    Consequently, it follows that
    $$
    \argmin_{1\le k\le K}\|F_{W_T}(x_i)-\mu_k\|_2=g(i)\quad\text{for all }i.
    $$
\end{theorem}

\begin{proof}
    Under \cref{assumption:exact_clustered}, we have $E=O$.
    By \cref{lem:ae_projection},
    $
    \|\Pi_{S_+}(E)\|_{2,\infty}\le\rho \varepsilon_o .
    $
    Applying \cref{lem:clean_meta} with $\nu=\rho \varepsilon_o $ and \cref{prop:path_perturb} give
    $
    \| \mathcal{F}_{\widetilde{W}_{T}}(\widetilde{X}) - \widetilde X\|_{2,\infty}\le2\rho \varepsilon_o
    $
    and
    $
    \|\mathcal{F}_{W_T}(X) - \mathcal{F}_{\widetilde{W}_{T}}(\widetilde{X}) \|_{2,\infty}\le\zeta_T.
    $
    Hence, we have that
    $$
    \|\mathcal{F}_{W_T}(X)-\widetilde X\|_{2,\infty}\le\zeta_T+2\rho \varepsilon_o ,
    $$
    which results in
    $
    \|F_{W_T}(x_i)-\mu_{g(i)}\|_2\le\zeta_T+2\rho \varepsilon_o,
    $
    for all $i \in [n].$
    
    Fix $i$ and write $k=g(i)$.
    For any $b\neq k$,
    \begin{align*}
        \|F_{W_T}(x_i)-\mu_b\|_2
        \ge
        \|\mu_b-\mu_k\|_2-\|F_{W_T}(x_i)-\mu_k\|_2
        \ge
        \delta-(\zeta_T+2\rho \varepsilon_o ).
    \end{align*}
    Since $\delta>2\zeta_T+4\rho \varepsilon_o $, we have
    $
    \delta-(\zeta_T+2\rho \varepsilon_o )>\zeta_T+2\rho \varepsilon_o \ge\|F_{W_T}(x_i)-\mu_k\|_2.
    $
    Therefore, we can conclude that
    $$
    \|F_{W_T}(x_i)-\mu_b\|_2>\|F_{W_T}(x_i)-\mu_k\|_2\quad\text{for all }b\neq k.
    $$
\end{proof}

The following corollary shows that, in the exact clustered case, the early memorization of cluster information directly leads to the IM effect in terms of reconstruction-error separation between inliers and outliers.

\begin{corollary}[Reconstruction-Error Separation in the Exact Clustered Case]
\label{cor:ae_outlier_exact}
Under the assumptions of \cref{thm:ae_exact}, every inlier $i\notin\mathcal{O}$ satisfies
$$ \|F_{W_T}(x_i)-x_i\|_2\le\zeta_T+2\rho \varepsilon_o. $$
On the other hand, every outlier $i\in\mathcal{O}$ satisfies
$$ \|F_{W_T}(x_i)-x_i\|_2\ge\|o_i\|_2-\zeta_T-2\rho \varepsilon_o. $$
Hence, if an outlier satisfies
$$ \|o_i\|_2>2\zeta_T+4\rho \varepsilon_o, $$
then its reconstruction error is strictly larger than that of every inlier.
\end{corollary}

\begin{proof}
    If $i\notin\mathcal{O}$, then $o_i=0$, and under \cref{assumption:exact_clustered},
    $
    x_i=\mu_{g(i)}.
    $
    Hence
    $
    \|F_{W_T}(x_i)-x_i\|_2=\|F_{W_T}(x_i)-\mu_{g(i)}\|_2\le\zeta_T+2\rho \varepsilon_o 
    $
    by \cref{thm:ae_exact}.
    
    If $i\in\mathcal{O}$, then
    $
    x_i=\mu_{g(i)}+o_i.
    $
    Therefore
    \begin{align*}
        \|F_{W_T}(x_i)-x_i\|_2
        &=
        \|F_{W_T}(x_i)-\mu_{g(i)}-o_i\|_2
        \ge
        \|o_i\|_2-\|F_{W_T}(x_i)-\mu_{g(i)}\|_2\\
        &\ge
        \|o_i\|_2-\zeta_T-2\rho \varepsilon_o .
    \end{align*}
    Thus, if $\|o_i\|_2>2\zeta_T+4\rho \varepsilon_o $, then
    $$
    \|F_{W_T}(x_i)-x_i\|_2>\zeta_T+2\rho \varepsilon_o ,
    $$
    whereas every inlier has reconstruction error at most $\zeta_T+2\rho \varepsilon_o $.
\end{proof}

\subsubsection{Main Results: IM Effect in the General Clustered Case $( \varepsilon_c >0)$}
\label{app_secA_real_case}

We next extend the previous result to the general clustered case $(\varepsilon_c>0)$, showing that the autoencoder still memorizes the underlying cluster structure early, up to an additional error induced by within-cluster variation.

\begin{theorem}[Early Memorization of Cluster Information in the General Case]
\label{thm:ae_beyond_exact}
    Assume Assumptions \ref{assumption:clustered_model}, \ref{assumption:outlier_rate}, \ref{assumption:outlier_magnitude}, \ref{assumption:center_jacobian}, \ref{assumption:radius_condition} and \ref{assumption:phi_bounds}.    
    Let $\alpha,\beta,L$ be the values set in \cref{prop:clean_jacobian}.
    Assume
    $
    0<\eta\le\min\left\{\frac{\alpha^2}{8\beta^4},\frac{\alpha^2}{2L\beta^2 {\color{black} \widehat R_0} }\right\}.
    $
    Fix an integer $T\ge T_\star( \varepsilon_c +\rho \varepsilon_o )$.
    If
    $
    \Delta_{\mathrm{path}}(T)\le R_{\mathrm{loc}}-\frac{2\beta}{\alpha^2} {\color{black} \widehat R_0} 
    $
    and
    $
    \delta:=\min_{a\neq b}\|\mu_a-\mu_b\|_2>2\zeta_T+4( \varepsilon_c +\rho \varepsilon_o ),
    $
    then, we have
    $$
    \|\mathcal{F}_{W_T}(X) - \widetilde X\|_{2,\infty}\le\zeta_T+2( \varepsilon_c +\rho \varepsilon_o ).
    $$
    Equivalently,
    $$
    \|F_{W_T}(x_i)-\mu_{g(i)}\|_2\le\zeta_T+2( \varepsilon_c +\rho \varepsilon_o )\quad\text{for all }i.
    $$
    Consequently,
    $$
    \argmin_{1\le k\le K}\|F_{W_T}(x_i)-\mu_k\|_2=g(i)\quad\text{for all }i.
    $$
\end{theorem}

\begin{proof}
    By the triangle inequality and \cref{lem:xi_projection,lem:ae_projection},
    $$
    \|\Pi_{S_+}(E)\|_{2,\infty}\le\|\Pi_{S_+}(\Xi)\|_{2,\infty}+\|\Pi_{S_+}(O)\|_{2,\infty}\le \varepsilon_c +\rho \varepsilon_o .
    $$
    Applying \cref{lem:clean_meta} with $\nu= \varepsilon_c +\rho \varepsilon_o $ gives
    $
    \| \mathcal{F}_{\widetilde{W}_{T}}(\widetilde{X}) -\widetilde X\|_{2,\infty}\le2( \varepsilon_c +\rho \varepsilon_o ).
    $
    and
    applying \cref{prop:path_perturb} gives
    $
    \| \mathcal{F}_{W_T}(X) - \mathcal{F}_{\widetilde{W}_{T}}(\widetilde{X}) \|_{2,\infty}\le\zeta_T.
    $
    Hence, we have
    $$
    \|\mathcal{F}_{W_T}(X) - \widetilde X\|_{2,\infty}\le\zeta_T+2( \varepsilon_c +\rho \varepsilon_o ),
    $$
    which concludes
    $$
    \|F_{W_T}(x_i)-\mu_{g(i)}\|_2\le\zeta_T+2( \varepsilon_c +\rho \varepsilon_o )\quad\text{for all }i.
    $$
    
    Fix $i$ and write $k=g(i)$.
    For any $b\neq k$,
    \begin{align*}
        \|F_{W_T}(x_i)-\mu_b\|_2
        \ge
        \|\mu_b-\mu_k\|_2-\|F_{W_T}(x_i)-\mu_k\|_2
        \ge
        \delta-(\zeta_T+2( \varepsilon_c +\rho \varepsilon_o )).
    \end{align*}
    Since $\delta>2\zeta_T+4( \varepsilon_c +\rho \varepsilon_o )$, we have
    $$
    \delta-(\zeta_T+2( \varepsilon_c +\rho \varepsilon_o ))>\zeta_T+2( \varepsilon_c +\rho \varepsilon_o )\ge\|F_{W_T}(x_i)-\mu_k\|_2.
    $$
    Therefore
    $$
    \|F_{W_T}(x_i)-\mu_b\|_2>\|F_{W_T}(x_i)-\mu_k\|_2\quad\text{for all }b\neq k.
    $$
\end{proof}

The following corollary presents our final main result: the early memorization of cluster information yields the IM effect even in the general clustered case $(\varepsilon_c>0)$, where reconstruction errors still separate inliers from sufficiently large outliers despite within-cluster variation.

\begin{corollary}[Reconstruction-Error Separation in the Exact Clustered Case]
\label{cor:ae_outlier_general}
    Under the assumptions and settings of \cref{thm:ae_beyond_exact}, every inlier $i\notin\mathcal{O}$ satisfies
    $$
    \|F_{W_T}(x_i)-x_i\|_2\le\zeta_T+3 \varepsilon_c +2\rho \varepsilon_o .
    $$
    On the other hand, every outlier $i\in\mathcal{O}$ satisfies
    $$
    \|F_{W_T}(x_i)-x_i\|_2\ge\|o_i\|_2-\zeta_T-3 \varepsilon_c -2\rho \varepsilon_o .
    $$
    Hence, if an outlier satisfies
    $$
    \|o_i\|_2>\delta+2\varepsilon_c,
    $$
    then its reconstruction error is strictly larger than that of every inlier.
\end{corollary}

\begin{proof}
    Fix $i$.
    If $i\notin\mathcal{O}$, then $o_i=0$, so
    $
    x_i=\mu_{g(i)}+\xi_i.
    $
    Therefore
    \begin{align*}
        \|F_{W_T}(x_i)-x_i\|_2
        &=
        \|F_{W_T}(x_i)-\mu_{g(i)}-\xi_i\|_2
        \le
        \|F_{W_T}(x_i)-\mu_{g(i)}\|_2+\|\xi_i\|_2\\
        &\le
        \zeta_T+2( \varepsilon_c +\rho \varepsilon_o )+ \varepsilon_c
        =
        \zeta_T+3 \varepsilon_c +2\rho \varepsilon_o.
    \end{align*}
    
    If $i\in\mathcal{O}$, then
    $
    x_i=\mu_{g(i)}+\xi_i+o_i.
    $
    Hence
    \begin{align*}
        \|F_{W_T}(x_i)-x_i\|_2
        &=
        \|F_{W_T}(x_i)-\mu_{g(i)}-\xi_i-o_i\|_2
        \ge
        \|o_i\|_2-\|F_{W_T}(x_i)-\mu_{g(i)}\|_2-\|\xi_i\|_2\\
        &\ge
        \|o_i\|_2-\zeta_T-2( \varepsilon_c +\rho \varepsilon_o )- \varepsilon_c
        =
        \|o_i\|_2-\zeta_T-3 \varepsilon_c -2\rho \varepsilon_o.
    \end{align*}
    Since \cref{thm:ae_beyond_exact} assumes
    $$
    \delta>2\zeta_T+4(\varepsilon_c+\rho\varepsilon_o),
    $$
    we obtain
    $$
    \delta+2\varepsilon_c>2\zeta_T+6\varepsilon_c+4\rho\varepsilon_o.
    $$
    Therefore, if
    $$
    \|o_i\|_2>\delta+2\varepsilon_c,
    $$
    then
    $$
    \|o_i\|_2>2\zeta_T+6 \varepsilon_c +4\rho \varepsilon_o,
    $$
    and hence
    $$
    \|o_i\|_2-\zeta_T-3 \varepsilon_c -2\rho \varepsilon_o
    >
    \zeta_T+3 \varepsilon_c +2\rho \varepsilon_o.
    $$
    Thus, every such outlier has reconstruction error strictly larger than that of every inlier.
\end{proof}

\subsubsection{Note: Monotonic Dependence of Key Quantities}
\label{app:dependence}

We summarize the dependence of the key constants appearing in the proofs. 
Throughout, we use the notation $c(\cdot)$ to indicate monotonic dependence.

We focus on the following elements:
$$
n_{\mathrm{min}},n_{\mathrm{max}},\varepsilon_c, \widetilde{R}_{0} ,\rho,\delta
$$
We use the notation $c = c(a_1\uparrow,\ldots,a_m\downarrow)$ to indicate that the constant $c$ depends on $a_1,\ldots,a_m$, where $\uparrow$ and $\downarrow$ denote that $c$ is increasing or decreasing in the corresponding argument. 
Logarithmic factors are ignored when specifying monotonicity.

The quantities $\alpha$, $\beta$, and $L$ scale with the cluster sizes as
\[
\alpha = c(n_{\min}\uparrow), \quad 
\beta = c(n_{\max}\uparrow), \quad 
L = c(n_{\max}\uparrow).
\]

The approximation and Jacobian errors satisfy
\[
\Delta_f = c(\varepsilon_c\uparrow), \quad 
\Delta_J = c(n\uparrow, \varepsilon_c\uparrow),
\]
which imply
\[
\beta_X = \beta + \Delta_J = c(n_{\max}\uparrow, \varepsilon_c\uparrow).
\]

The key growth parameters $\kappa$ and $\gamma$ satisfy
\[
\kappa = c(n_{\max}\uparrow, \varepsilon_c\uparrow, \widetilde{R}_{0}\uparrow), \quad
\gamma = c(n_{\max}\uparrow, \varepsilon_c\uparrow, \widetilde{R}_{0}\uparrow).
\]

Consequently, the path deviation and residual bounds satisfy
\[
\Delta_{\mathrm{path}}(T) = c(T, n_{\max}\uparrow, \varepsilon_c\uparrow, \widetilde{R}_{0}\uparrow), \quad
\zeta_T = c(T, n_{\max}\uparrow, \varepsilon_c\uparrow, \widetilde{R}_{0}\uparrow).
\]

The lower bound on the iteration index in the main theorem satisfies
\[
T_\star(\varepsilon_c + \rho \varepsilon_o) = c(n_{\min}\downarrow),
\]
while the admissible range for the upper bound is determined by the conditions
\[
\Delta_{\mathrm{path}}(T) \le R_{\mathrm{loc}} - \frac{2\beta}{\alpha^2}\widehat R_0,
\quad
\delta > 2\zeta_T + 4(\varepsilon_c + \rho \varepsilon_o),
\]
which yield upper bounds of the form
\[
c(n_{\min}\uparrow, n_{\max}\downarrow, \varepsilon_c\downarrow, \widetilde{R}_{0}\downarrow, \rho\downarrow, \delta\uparrow)
\]
since $\widehat{R}_{0}=\widetilde{R}_{0}+\sqrt n \varepsilon_c+\sqrt{\rho n} \varepsilon_o=c(\epsilon_c\uparrow,\widetilde{R}_{0}\uparrow, \rho\uparrow)$.
Therefore, the lower and upper bounds, i.e., $T_1$ and $T_2$, on the iteration interval in the main theorem scale as
\[
c(n_{\min}\downarrow)
\quad \text{and} \quad
c(n_{\min}\uparrow, n_{\max}\downarrow, \varepsilon_c\downarrow, \widetilde{R}_{0}\downarrow, \rho\downarrow, \delta\uparrow),
\]
respectively.


\clearpage
\section{Detailed Experimental Analysis}
\label{sec-appen:exp}

\subsection{Data Description}
\label{sec-appen:data}
\paragraph{Simulation Datasets}
We generate a synthetic dataset consisting of $n = 2{,}000$ samples in $\mathbb{R}^2$ with $K = 3$ cluster centers. 
The cluster centers $\mu_1, \mu_2,$ and $\mu_3$ are placed in $[-10, 10]^2$ such that the minimum pairwise separation satisfies $\delta = 16$, ensuring well-separated clusters. 

Inlier samples are distributed approximately uniformly across the three clusters (about $633$ samples per cluster) unless otherwise specified. 
Each inlier $x_i$ is generated as $x_i = \mu_{g(i)} + \xi_i$, where $g(i) \in \{1,2,3\}$ denotes the cluster assignment and $\xi_i \sim \mathcal{N}(0, \widetilde{\varepsilon}_c^2 I)$ is Gaussian noise. 
The parameter $\widetilde{\varepsilon}_c$ controls the within-cluster variation, corresponding to $\varepsilon_c$ in \cref{assumption:clustered_model}. 

Outliers are generated independently with an outlier rate $\rho = 0.05$, resulting in $100$ outlier samples. 
These samples are drawn uniformly from $[-18, 18]^2$, subject to the constraint that each outlier is at least a distance of $10$ away from every cluster center, ensuring that they lie outside the inlier regions.

Finally, all features are min-max scaled to the range $[0, 1/\sqrt{2}]$, so that the resulting inputs have bounded $\ell_2$ norms not exceeding one.

While the above describes the baseline data generation process, the parameters $\varepsilon_c$, $(n_{\min}, n_{\max})$, and $\rho$ are varied across simulation studies to examine their effects. 
We refer to \cref{sec:exp_simulation} in the main manuscript for detailed configurations.


\paragraph{Real Benchmark Datasets}
We evaluate a total of 57 datasets from \texttt{ADBench} \citep{han2022adbench}, including 46 tabular datasets, 6 image datasets, and 5 text datasets. 
The tabular datasets cover diverse application domains, including healthcare, finance, physical, and astronautics. 
Table~\ref{tab:adbench_summ} summarizes the basic information of all datasets.
For additional details, we refer the reader to the official \texttt{ADBench} GitHub repository\footnote{\url{https://github.com/Minqi824/ADBench}}.

\begin{table}[h!]
\renewcommand\thetable{B.1}
\centering
\caption{Description of \texttt{ADBench} benchmark datasets.
}
\label{tab:adbench_summ}
\vskip 0.1in
\setlength{\tabcolsep}{1mm}
\fontsize{8pt}{8pt}\selectfont
\resizebox{0.85\textwidth}{!}{
\begin{tabular}{l>{\ttfamily}lccccl}
\toprule
\textbf{Number} & {\normalfont{\textbf{Dataset Name}}} & \textbf{\#Samples} & \textbf{\#Features} & \textbf{\#Anomaly} & \textbf{\%Anomaly} & \textbf{Category} \\
\midrule
\textbf{1}      & ALOI                  & 49534              & 27                  & 1508               & 3.04               & Image             \\
\textbf{2}      & annthyroid            & 7200               & 6                   & 534                & 7.42               & Healthcare        \\
\textbf{3}      & backdoor              & 95329              & 196                 & 2329               & 2.44               & Network           \\
\textbf{4}      & breastw               & 683                & 9                   & 239                & 34.99              & Healthcare        \\
\textbf{5}      & campaign              & 41188              & 62                  & 4640               & 11.27              & Finance           \\
\textbf{6}      & cardio                & 1831               & 21                  & 176                & 9.61               & Healthcare        \\
\textbf{7}      & Cardiotocography      & 2114               & 21                  & 466                & 22.04              & Healthcare        \\
\textbf{8}      & celeba                & 202599             & 39                  & 4547               & 2.24               & Image             \\
\textbf{9}      & census                & 299285             & 500                 & 18568              & 6.2                & Sociology         \\
\textbf{10}     & cover                 & 286048             & 10                  & 2747               & 0.96               & Botany            \\
\textbf{11}     & donors                & 619326             & 10                  & 36710              & 5.93               & Sociology         \\
\textbf{12}     & fault                 & 1941               & 27                  & 673                & 34.67              & Physical          \\
\textbf{13}     & fraud                 & 284807             & 29                  & 492                & 0.17               & Finance           \\
\textbf{14}     & glass                 & 214                & 7                   & 9                  & 4.21               & Forensic          \\
\textbf{15}     & Hepatitis             & 80                 & 19                  & 13                 & 16.25              & Healthcare        \\
\textbf{16}     & http                  & 567498             & 3                   & 2211               & 0.39               & Web               \\
\textbf{17}     & InternetAds           & 1966               & 1555                & 368                & 18.72              & Image             \\
\textbf{18}     & Ionosphere            & 351                & 32                  & 126                & 35.9               & Oryctognosy       \\
\textbf{19}     & landsat               & 6435               & 36                  & 1333               & 20.71              & Astronautics      \\
\textbf{20}     & letter                & 1600               & 32                  & 100                & 6.25               & Image             \\
\textbf{21}     & Lymphography          & 148                & 18                  & 6                  & 4.05               & Healthcare        \\
\textbf{22}     & magic.gamma           & 19020              & 10                  & 6688               & 35.16              & Physical          \\
\textbf{23}     & mammography           & 11183              & 6                   & 260                & 2.32               & Healthcare        \\
\textbf{24}     & mnist                 & 7603               & 100                 & 700                & 9.21               & Image             \\
\textbf{25}     & musk                  & 3062               & 166                 & 97                 & 3.17               & Chemistry         \\
\textbf{26}     & optdigits             & 5216               & 64                  & 150                & 2.88               & Image             \\
\textbf{27}     & PageBlocks            & 5393               & 10                  & 510                & 9.46               & Document          \\
\textbf{28}     & pendigits             & 6870               & 16                  & 156                & 2.27               & Image             \\
\textbf{29}     & Pima                  & 768                & 8                   & 268                & 34.9               & Healthcare        \\
\textbf{30}     & satellite             & 6435               & 36                  & 2036               & 31.64              & Astronautics      \\
\textbf{31}     & satimage-2            & 5803               & 36                  & 71                 & 1.22               & Astronautics      \\
\textbf{32}     & shuttle               & 49097              & 9                   & 3511               & 7.15               & Astronautics      \\
\textbf{33}     & skin                  & 245057             & 3                   & 50859              & 20.75              & Image             \\
\textbf{34}     & smtp                  & 95156              & 3                   & 30                 & 0.03               & Web               \\
\textbf{35}     & SpamBase              & 4207               & 57                  & 1679               & 39.91              & Document          \\
\textbf{36}     & speech                & 3686               & 400                 & 61                 & 1.65               & Linguistics       \\
\textbf{37}     & Stamps                & 340                & 9                   & 31                 & 9.12               & Document          \\
\textbf{38}     & thyroid               & 3772               & 6                   & 93                 & 2.47               & Healthcare        \\
\textbf{39}     & vertebral             & 240                & 6                   & 30                 & 12.5               & Biology           \\
\textbf{40}     & vowels                & 1456               & 12                  & 50                 & 3.43               & Linguistics       \\
\textbf{41}     & Waveform              & 3443               & 21                  & 100                & 2.9                & Physics           \\
\textbf{42}     & WBC                   & 223                & 9                   & 10                 & 4.48               & Healthcare        \\
\textbf{43}     & WDBC                  & 367                & 30                  & 10                 & 2.72               & Healthcare        \\
\textbf{44}     & Wilt                  & 4819               & 5                   & 257                & 5.33               & Botany            \\
\textbf{45}     & wine                  & 129                & 13                  & 10                 & 7.75               & Chemistry         \\
\textbf{46}     & WPBC                  & 198                & 33                  & 47                 & 23.74              & Healthcare        \\
\textbf{47}     & yeast                 & 1484               & 8                   & 507                & 34.16              & Biology           \\
\textbf{48}     & CIFAR10               & 5263               & 512                 & 263                & 5                  & Image             \\
\textbf{49}     & FashionMNIST          & 6315               & 512                 & 315                & 5                  & Image             \\
\textbf{50}     & MNIST-C               & 10000              & 512                 & 500                & 5                  & Image             \\
\textbf{51}     & MVTec-AD              & 5354               & 512                 & 1258               & 23.5               & Image             \\
\textbf{52}     & SVHN                  & 5208               & 512                 & 260                & 5                  & Image             \\
\textbf{53}     & Agnews                & 10000              & 768                 & 500                & 5                  & NLP               \\
\textbf{54}     & Amazon                & 10000              & 768                 & 500                & 5                  & NLP               \\
\textbf{55}     & Imdb                  & 10000              & 768                 & 500                & 5                  & NLP               \\
\textbf{56}     & Yelp                  & 10000              & 768                 & 500                & 5                  & NLP               \\
\textbf{57}     & 20news        & 11905              & 768                 & 591                & 4.96               & NLP              \\
\bottomrule
\end{tabular}
}
\end{table}

All input features, including both raw features and learned representations, are preprocessed using min-max scaling prior to model training. 
For tabular datasets without learned representations, we directly use the scaled raw features as inputs.

For image and text datasets, we directly use the pre-computed embedding features provided in \texttt{ADBench}. 
Specifically, image datasets are represented using features extracted from a pre-trained Vision Transformer (ViT, \cite{DBLP:conf/iclr/DosovitskiyB0WZ21}), and text datasets are represented using embeddings obtained via BERT \cite{DBLP:conf/naacl/DevlinCLT19}.

For tabular datasets, we construct representations using TabPFN-v2 \cite{DBLP:journals/corr/abs-2502-17361} obtained from the official GitHub repository\footnote{\url{https://github.com/PriorLabs/TabPFN}}.
Since TabPFN requires a supervised setup, we treat the first column as the target variable and the remaining columns as input features, and extract representations based on the resulting model.
When the first column has negligible variance, we instead use the second column as the target.
To ensure computational feasibility and representation quality, we restrict our analysis to datasets satisfying $4{,}000 < n < 50{,}000$ and $p < 50$.
Under these criteria, the tabular datasets used for representation learning are:
\texttt{ALOI}, \texttt{annthyroid}, \texttt{campaign}, \texttt{landsat}, \texttt{magic.gamma}, \texttt{mammography}, \texttt{PageBlocks}, \texttt{shuttle}, \texttt{SpamBase}, \texttt{Wilt}, and \texttt{yeast}.

\subsection{Implementation Details}
\label{sec-appen:implementation}

As described in \cref{sec-appen:data}, all datasets are preprocessed using min-max scaling so that each feature lies in $[0, 1]$. 
All experiments are run on a single NVIDIA H100 GPU using \texttt{PyTorch}. 
Each reported result is averaged over five independent runs with different random initializations and mini-batch orderings.
For the simulation experiments involving $\widetilde{R}_0$, randomness is introduced only through the mini-batch ordering, while the initialization is kept fixed. 
The remainder of this subsection details the architecture and optimization choices for the simulation study (\cref{sec:exp_simulation}) and the real-world experiments (\cref{sec:exp-real}).


\clearpage
\paragraph{Simulation Study}

We train a single-hidden-layer autoencoder $F_W(x) = A \tanh(W x)$ with hidden size $H = 32.$
Both the encoder weight $W$ and the decoder weight $A$ are trainable and initialized using the Kaiming uniform scheme \citep{7410480}.
Training uses the Adam optimizer \citep{kingma2014adam} with a constant learning rate of $2 \times 10^{-4}$, batch size $500$, and $1{,}000$ epochs (i.e., $4{,}000$ optimizer updates given $n = 2{,}000$).
The IM window is computed as the maximal contiguous interval of training epochs $\tau$ during which $\textup{AUROC}(\tau) \ge 0.9$.
For experiment (iv), the three initializations with different $\widetilde{R}_0$ are obtained by sampling three random initializations using the Kaiming uniform scheme \citep{7410480}.
All simulation results are averaged over three random seeds.

\paragraph{Real-World Experiments: Backbones and Baselines}

For the two IM-based backbones, ODIM \citep{DBLP:conf/icml/KimHLKK24} and ALTBI \citep{DBLP:conf/aaai/ChoHBK25}, we follow the original source codes published by their authors, including the network architectures, optimizers, batch sizes, and stopping criteria. 
Since ODIM and ALTBI share the same implementation details, we adopt the configuration provided in the official ODIM GitHub repository.\footnote{\url{https://github.com/jshwang0311/ODIM}} 
Unless otherwise specified, we use min-max scaled raw input features for tabular data, ViT-based representations followed by min-max scaling for image data, and BERT-based representations followed by min-max scaling for text data.

At the beginning of each training run, we maintain an exponential moving average (EMA) of the network parameters,
$\bar{W}_{t} = \alpha\, \bar{W}_{t-1} + (1-\alpha)\, W_{t}$,
with decay $\alpha = 0.999$ over the first $50$ optimizer updates.
The averaged parameters $\bar{W}_{50}$ are then used as the initial weights for the main ALTBI/ODIM training run, after which the standard (non-EMA) optimizer dynamics resume.
A decay of $0.999$ corresponds to an effective averaging window of order $10^{3}$ iterations, which is long enough that the warm-up phase emphasizes the slowly-changing inlier-driven component of the gradient and suppresses the irregular signal contributed by outliers, consistent with the role of $\widetilde{R}_0$ identified in \cref{thm_main}.

The modified ODIM and ALTBI with two proposed strategies (denoted ODIM* and ALTBI*, respectively, in \cref{fig:real_auroc} and \cref{fig:real_auprc}) replace the input with a pre-trained representation (TabPFN for tabular data, ViT for image data) and warm up the parameters with the EMA scheme.

The 22 baselines cover classical machine-learning detectors, deep autoencoder-based methods, and diffusion-based detectors. 
These methods include traditional machine-learning approaches implemented in \texttt{ADBench}, such as kNN \cite{ramaswamy2000efficient}, LOF \cite{10.1145/335191.335388}, OCSVM \cite{ocsvm}, CBLOF \cite{he2003discovering}, PCA \cite{shyu2003novel}, FeatureBagging \cite{lazarevic2005feature}, IForest \citep{liu2008isolation}, MCD \cite{fauconnier2009outliers}, HBOS \cite{goldstein2012histogram}, LODA \cite{pevny2016loda}, COPOD \cite{li2020copod}, and ECOD \cite{li2022ecod}.
We also consider two deep learning-based UOD methods, DAGMM \cite{dagmm} and DeepSVDD \citep{deepsvdd}, which are available in \texttt{ADBench}, as well as more recent approaches beyond \texttt{ADBench}, including DROCC \cite{DBLP:conf/icml/GoyalRJS020}, ICL \citep{icl}, GOAD \cite{goad}, and DTE \cite{DBLP:journals/corr/abs-2305-18593}.

For fair comparison, the results of the 22 baseline methods are directly taken from the supplementary materials of \cite{DBLP:journals/corr/abs-2305-18593}, without re-running the corresponding models. 
In contrast, we implement ODIM, ALTBI, and their modified versions, by ourselves.

\clearpage
\subsection{Detailed Experiment Results}
\label{sec-append:futehr_exp}

{\color{black}

\textcolor{black}{
This section supports \cref{sec:exp-real} with three additional results:
(i) an ablation study that disentangles the effect of EMA warm-up from that of the pre-trained representation,
(ii) the AUPRC results, while the main body shows AUROC results, and
(iii) a full benchmark comparison on \texttt{ADBench} evaluating the proposed guidelines across all datasets.
}

\paragraph{\textcolor{black}{Ablation Study Results}}

\cref{tab:ablation_tabular} reports the AUROC of IM-based methods, ALTBI and ODIM, on the 47 tabular \texttt{ADBench} datasets under three configurations:
(i) the vanilla raw input training without the proposed techniques, corresponding to the original implementations of ODIM and ALTBI,
(ii) with the EMA warm-up applied to early iterations ($+$EMA), and
(iii) with both the EMA warm-up and a TabPFN embedding ($+$EMA $+$TabPFN), where the TabPFN representation is applied only to the subset of datasets described in \cref{sec-appen:data}.

The EMA warm-up alone improves the AUROC of both methods by 0.6 to 0.7 points, supporting the role of small $\widetilde{R}_0$ in widening the IM window.
Using the TabPFN representation gives a further AUROC improvement for ALTBI, while for ODIM the two strategies yield comparable AUROC.

}

\begin{table}[h]
    \renewcommand\thetable{B.2}
    \centering
    \small
    \caption{\textcolor{black}{
    Comparison of performance on 47 tabular \texttt{ADBench} datasets \citep{han2022adbench} in terms of the two proposed strategies.
    Best results are \textbf{bolded}.
    }}
    \label{tab:ablation_tabular}
    \vskip 0.1in
    \begin{tabular}{l cc}
        \toprule
        Method                       & ALTBI          & ODIM           \\
        \midrule
        Raw (Vanilla)                & 0.764          & 0.759          \\
        $+$EMA                      & 0.770          & \textbf{0.766} \\
        $+$EMA $+$TabPFN (ours)   & \textbf{0.775} & \textbf{0.766} \\
        \bottomrule
    \end{tabular}
\end{table}

{\color{black}
\cref{tab:ablation_image} reports the analogous comparison on the three image datasets.
As in the tabular case, `Raw (Vanilla)' denotes training on raw inputs without EMA warm-up, 
`$+$EMA' applies EMA to raw inputs, 
and `$+$EMA + ViT' further combines EMA with the ViT representation.
Using EMA alone does not provide a significant improvement. 
In contrast, `$+$EMA $+$ViT' improves the performance of ALTBI on \texttt{CIFAR10} and \texttt{FashionMNIST}, and consistently improves the performance of ODIM on these datasets, while yielding comparable AUROC on \texttt{SVHN}. 
This is because, on \texttt{CIFAR10} and \texttt{FashionMNIST}, the ViT embedding tightens the inlier cluster structure, allowing the EMA warm-up to further enlarge the IM window. 
In contrast, on \texttt{SVHN}, the ViT embedding does not produce sufficiently compact inlier clusters for the EMA warm-up to yield additional gains.
}

\begin{table}[h]
    \renewcommand\thetable{B.3}
    \centering
    \small
    \caption{\textcolor{black}{
    Per-component ablation on the three image benchmarks, in AUROC averaged over the ten normal-class settings.
    `$+$ViT' applies the original ALTBI \citep{DBLP:conf/aaai/ChoHBK25} or ODIM \citep{DBLP:conf/icml/KimHLKK24} method on a ViT embedding \citep{DBLP:conf/iclr/DosovitskiyB0WZ21}, and `$+$EMA $+$ViT' additionally adds the EMA warm-up.
    Best results are \textbf{bolded}.
    }}
    \label{tab:ablation_image}
    \vskip 0.1in
    \begin{tabular}{l ccc ccc}
        \toprule
        \multirow{2}{*}{Method} & \multicolumn{3}{c}{ALTBI} & \multicolumn{3}{c}{ODIM} \\
        \cmidrule(lr){2-4} \cmidrule(lr){5-7}
        & \texttt{CIFAR10}        & \texttt{FashionMNIST}   & \texttt{SVHN}           & \texttt{CIFAR10}       & \texttt{FashionMNIST}   & \texttt{SVHN}           \\
        \midrule
         {Raw (Vanilla)}     & {0.576} & {0.901} & {\textbf{0.592}} & {0.578} & {0.893} & {0.577} \\
        $+$EMA                       & 0.576 & 0.902 & \textbf{0.592} & 0.579 & 0.892 & 0.576 \\
        $+$EMA $+$ViT (ours)        & \textbf{0.888} & \textbf{0.920} & 0.581 & \textbf{0.876} & \textbf{0.902} & \textbf{0.579} \\
        \bottomrule
    \end{tabular}
\end{table}

\clearpage
\paragraph{AUPRC Results}

\Cref{fig:real_auprc} presents the performance comparison under the AUPRC metric.
On the tabular datasets (leftmost panel), ALTBI* and ODIM* are still among the top-performing methods, although the gap to machine-learning-based baselines such as MCD and IForest narrows compared to AUROC.
This result reflects the well-known sensitivity of AUPRC to the anomaly proportion, since on a subset of \texttt{ADBench} datasets with extremely low contamination rates, isolation- and density-based detectors can retain a relative advantage.
On the image datasets, the improvement from the pre-trained representation is even more pronounced under AUPRC than under AUROC.
That is, ALTBI* and ODIM* dominate all baselines on CIFAR10 and FashionMNIST by a large margin.

\begin{figure}[h]
    \centering
    \includegraphics[width=0.99\linewidth]{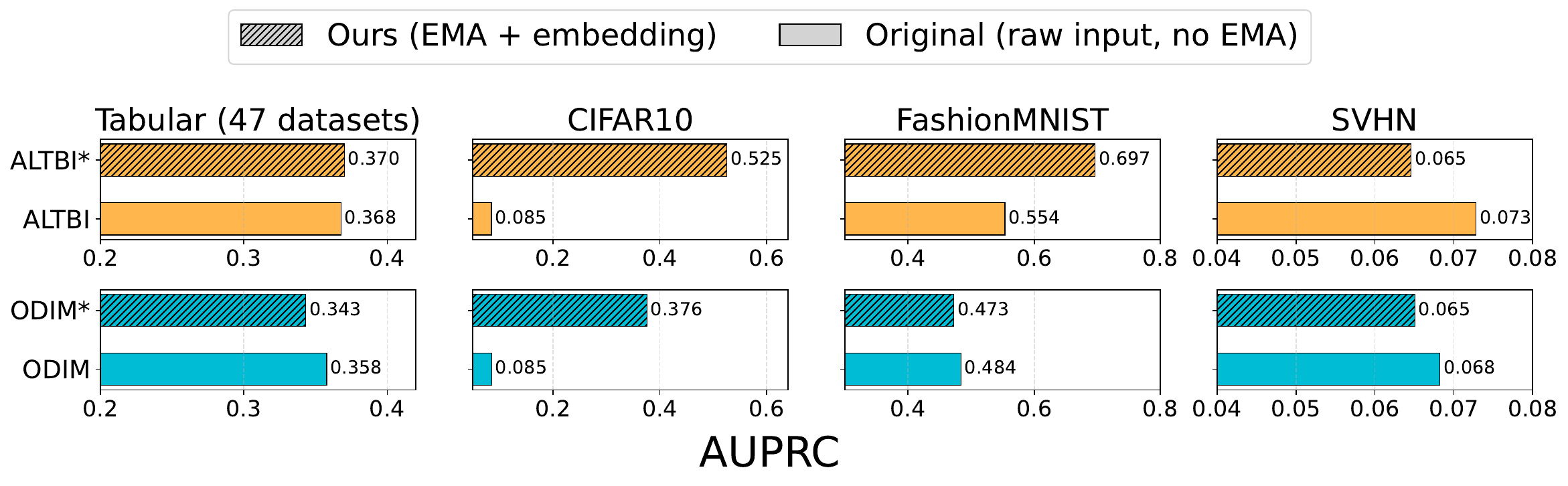}
    \caption{\textcolor{black}{
    Unsupervised outlier detection performance (AUPRC) on the \texttt{ADBench} benchmark \citep{han2022adbench}.
    (Left $\to$ Right) average over 47 tabular datasets and results on CIFAR10, FashionMNIST, and SVHN.
    Hatched bars (ALTBI* and ODIM*) denote the proposed variants that combine the EMA-based initialization with a pre-trained representation.
    }}
    \label{fig:real_auprc}
\end{figure}

\paragraph{Full Benchmark Results}
{\color{black}
\cref{tab:full_auroc} and \cref{tab:full_auprc} compare the IM-based methods (ALTBI and ODIM) on all 57 \texttt{ADBench} datasets, under AUROC and AUPRC.
In particular, we compare three variants:
Original (the original ALTBI and ODIM using \texttt{ADBench}'s default inputs for each data type, i.e., raw features for tabular data and ViT/BERT embeddings for image/text data),
$+$EMA (applying EMA warm-up on the same inputs), and
$+$EMA $+$TabPFN (additionally replacing a subset of tabular inputs with TabPFN embeddings).

The proposed guidelines consistently improve performance over the original setting. The average AUROC of ALTBI rises from $0.757$ (Original) to $0.762$ ($+$EMA) and further to $0.766$ (Ours), and a similar trend is observed for ODIM ($0.751 \to 0.756 \to 0.757$). Analogous improvements also appear for AUPRC, confirming consistent performance improvement across diverse data types.

\cref{tab:adbench1,tab:adbench2} show comparisons between our proposed method and 22 existing baselines, including both classical machine-learning and deep-learning approaches. 
The baseline results are directly taken from \citep{DBLP:journals/corr/abs-2305-18593}, whereas ALTBI and ODIM are run by our own implementations. 
All reported results are averaged over three random initializations.
We note that, since the original ALTBI and ODIM already employ embedding representations for image and text datasets, our modifications mainly consist of introducing the EMA warm-up and partially replacing tabular raw inputs with TabPFN embeddings. 
In \cref{tab:adbench1,tab:adbench2}, we report only the modified variant of ALTBI.
We can observe that ALTBI with our proposed approaches achieves the best performance in terms of both AUROC and AUPRC.
}

\begin{table}[h!]
\renewcommand\thetable{B.4}
\centering
\caption{
\textcolor{black}{Per-dataset AUROC of the variants of the IM-based methods (Original, $+$EMA, Ours) on the 57 \texttt{ADBench} datasets.
Best average AUROC for each method is \textbf{bolded}.}
}
\label{tab:full_auroc}
\vskip 0.15in
\setlength{\tabcolsep}{1mm}
\fontsize{6pt}{6pt}\selectfont
\resizebox{0.8\textwidth}{!}{
\begin{tabular}{l>{\ttfamily}lccc|ccc}
\toprule
\multirow{2}{*}{\textbf{Number}} & \multirow{2}{*}{\normalfont{\textbf{Dataset Name}}} & \multicolumn{3}{c}{\textbf{ALTBI}} & \multicolumn{3}{c}{\textbf{ODIM}} \\
\cmidrule(lr){3-5} \cmidrule(lr){6-8}
 & & {Original} & {$+$EMA} & {Ours} & {Original} & {$+$EMA} & {Ours} \\
\midrule
\textbf{1} & ALOI & 0.521 & 0.526 & 0.601 & 0.525 & 0.530 & 0.606 \\
\textbf{2} & annthyroid & 0.633 & 0.642 & 0.691 & 0.632 & 0.625 & 0.685 \\
\textbf{3} & backdoor & 0.876 & 0.868 & 0.868 & 0.909 & 0.897 & 0.897 \\
\textbf{4} & breastw & 0.981 & 0.982 & 0.982 & 0.980 & 0.978 & 0.978 \\
\textbf{5} & campaign & 0.726 & 0.726 & 0.726 & 0.729 & 0.735 & 0.735 \\
\textbf{6} & cardio & 0.845 & 0.840 & 0.840 & 0.869 & 0.876 & 0.876 \\
\textbf{7} & Cardiotocography & 0.540 & 0.530 & 0.530 & 0.542 & 0.534 & 0.534 \\
\textbf{8} & celeba & 0.804 & 0.797 & 0.797 & 0.782 & 0.760 & 0.760 \\
\textbf{9} & census & 0.670 & 0.672 & 0.672 & 0.682 & 0.685 & 0.685 \\
\textbf{10} & cover & 0.897 & 0.910 & 0.910 & 0.920 & 0.914 & 0.914 \\
\textbf{11} & donors & 0.609 & 0.650 & 0.650 & 0.606 & 0.624 & 0.624 \\
\textbf{12} & fault & 0.661 & 0.658 & 0.658 & 0.678 & 0.671 & 0.671 \\
\textbf{13} & fraud & 0.920 & 0.925 & 0.925 & 0.913 & 0.920 & 0.920 \\
\textbf{14} & glass & 0.748 & 0.771 & 0.771 & 0.757 & 0.769 & 0.769 \\
\textbf{15} & Hepatitis & 0.811 & 0.818 & 0.818 & 0.703 & 0.719 & 0.719 \\
\textbf{16} & http & 0.996 & 0.996 & 0.996 & 0.996 & 0.996 & 0.996 \\
\textbf{17} & InternetAds & 0.733 & 0.745 & 0.745 & 0.731 & 0.734 & 0.734 \\
\textbf{18} & Ionosphere & 0.850 & 0.883 & 0.883 & 0.862 & 0.860 & 0.860 \\
\textbf{19} & landsat & 0.527 & 0.539 & 0.569 & 0.499 & 0.545 & 0.600 \\
\textbf{20} & letter & 0.722 & 0.733 & 0.733 & 0.730 & 0.736 & 0.736 \\
\textbf{21} & Lymphography & 0.988 & 0.988 & 0.988 & 0.980 & 0.994 & 0.994 \\
\textbf{22} & magic.gamma & 0.746 & 0.745 & 0.694 & 0.724 & 0.727 & 0.672 \\
\textbf{23} & mammography & 0.808 & 0.814 & 0.763 & 0.813 & 0.816 & 0.774 \\
\textbf{24} & mnist & 0.796 & 0.794 & 0.794 & 0.791 & 0.779 & 0.779 \\
\textbf{25} & musk & 1.000 & 1.000 & 1.000 & 0.971 & 0.945 & 0.945 \\
\textbf{26} & optdigits & 0.602 & 0.684 & 0.684 & 0.552 & 0.697 & 0.697 \\
\textbf{27} & PageBlocks & 0.868 & 0.882 & 0.807 & 0.864 & 0.861 & 0.815 \\
\textbf{28} & pendigits & 0.939 & 0.957 & 0.957 & 0.911 & 0.918 & 0.918 \\
\textbf{29} & Pima & 0.716 & 0.713 & 0.713 & 0.712 & 0.718 & 0.718 \\
\textbf{30} & satellite & 0.734 & 0.760 & 0.760 & 0.702 & 0.740 & 0.740 \\
\textbf{31} & satimage-2 & 0.997 & 0.998 & 0.998 & 0.997 & 0.997 & 0.997 \\
\textbf{32} & shuttle & 0.983 & 0.984 & 0.985 & 0.982 & 0.981 & 0.709 \\
\textbf{33} & skin & 0.869 & 0.865 & 0.865 & 0.873 & 0.877 & 0.877 \\
\textbf{34} & smtp & 0.823 & 0.845 & 0.845 & 0.804 & 0.840 & 0.840 \\
\textbf{35} & SpamBase & 0.491 & 0.512 & 0.670 & 0.512 & 0.513 & 0.667 \\
\textbf{36} & speech & 0.471 & 0.468 & 0.468 & 0.471 & 0.469 & 0.469 \\
\textbf{37} & Stamps & 0.878 & 0.883 & 0.883 & 0.898 & 0.896 & 0.896 \\
\textbf{38} & thyroid & 0.938 & 0.939 & 0.939 & 0.938 & 0.942 & 0.942 \\
\textbf{39} & vertebral & 0.385 & 0.373 & 0.373 & 0.374 & 0.360 & 0.360 \\
\textbf{40} & vowels & 0.878 & 0.882 & 0.882 & 0.862 & 0.874 & 0.874 \\
\textbf{41} & Waveform & 0.749 & 0.713 & 0.713 & 0.763 & 0.780 & 0.780 \\
\textbf{42} & WBC & 0.994 & 0.993 & 0.993 & 0.992 & 0.993 & 0.993 \\
\textbf{43} & WDBC & 0.985 & 0.981 & 0.981 & 0.993 & 0.986 & 0.986 \\
\textbf{44} & Wilt & 0.349 & 0.375 & 0.409 & 0.364 & 0.382 & 0.401 \\
\textbf{45} & wine & 0.948 & 0.940 & 0.940 & 0.875 & 0.863 & 0.863 \\
\textbf{46} & WPBC & 0.506 & 0.503 & 0.503 & 0.502 & 0.526 & 0.526 \\
\textbf{47} & yeast & 0.411 & 0.411 & 0.469 & 0.398 & 0.409 & 0.471 \\
\textbf{48} & CIFAR10 & 0.889 & 0.888 & 0.888 & 0.877 & 0.876 & 0.876 \\
\textbf{49} & FashionMNIST & 0.921 & 0.920 & 0.920 & 0.902 & 0.902 & 0.902 \\
\textbf{50} & MNIST-C & 0.739 & 0.741 & 0.741 & 0.726 & 0.727 & 0.727 \\
\textbf{51} & MVTec-AD & 0.867 & 0.860 & 0.860 & 0.849 & 0.845 & 0.845 \\
\textbf{52} & SVHN & 0.581 & 0.581 & 0.581 & 0.579 & 0.579 & 0.579 \\
\textbf{53} & Agnews & 0.868 & 0.864 & 0.864 & 0.862 & 0.862 & 0.862 \\
\textbf{54} & Amazon & 0.546 & 0.548 & 0.548 & 0.544 & 0.539 & 0.539 \\
\textbf{55} & Imdb & 0.523 & 0.527 & 0.527 & 0.524 & 0.526 & 0.526 \\
\textbf{56} & Yelp & 0.555 & 0.557 & 0.557 & 0.558 & 0.554 & 0.554 \\
\textbf{57} & 20news & 0.725 & 0.725 & 0.725 & 0.725 & 0.722 & 0.722 \\
\midrule
\multicolumn{2}{c}{\textbf{Average}} & 0.757 & 0.762 & \textbf{0.766} & 0.751 & 0.756 & \textbf{0.757} \\
\bottomrule
\end{tabular}
}
\end{table}

\begin{table}[h!]
\renewcommand\thetable{B.5}
\centering
\caption{
\textcolor{black}{Per-dataset AUPRC of the variants of the IM-based methods (Original, $+$EMA, Ours) on the 57 \texttt{ADBench} datasets.
Best average AUPRC for each method is \textbf{bolded}.}
}
\label{tab:full_auprc}
\vskip 0.15in
\setlength{\tabcolsep}{1mm}
\fontsize{6pt}{6pt}\selectfont
\resizebox{0.8\textwidth}{!}{
\begin{tabular}{l>{\ttfamily}lccc|ccc}
\toprule
\multirow{2}{*}{\textbf{Number}} & \multirow{2}{*}{\normalfont{\textbf{Dataset Name}}} & \multicolumn{3}{c}{\textbf{ALTBI}} & \multicolumn{3}{c}{\textbf{ODIM}} \\
\cmidrule(lr){3-5} \cmidrule(lr){6-8}
 & & {Original} & {$+$EMA} & {Ours} & {Original} & {$+$EMA} & {Ours} \\
\midrule
\textbf{1} & ALOI & 0.035 & 0.035 & 0.057 & 0.037 & 0.037 & 0.064 \\
\textbf{2} & annthyroid & 0.156 & 0.155 & 0.181 & 0.154 & 0.154 & 0.192 \\
\textbf{3} & backdoor & 0.138 & 0.122 & 0.122 & 0.406 & 0.396 & 0.396 \\
\textbf{4} & breastw & 0.943 & 0.952 & 0.952 & 0.930 & 0.915 & 0.915 \\
\textbf{5} & campaign & 0.266 & 0.247 & 0.247 & 0.246 & 0.243 & 0.243 \\
\textbf{6} & cardio & 0.539 & 0.519 & 0.519 & 0.526 & 0.513 & 0.513 \\
\textbf{7} & Cardiotocography & 0.424 & 0.409 & 0.409 & 0.390 & 0.402 & 0.402 \\
\textbf{8} & celeba & 0.098 & 0.090 & 0.090 & 0.064 & 0.054 & 0.054 \\
\textbf{9} & census & 0.090 & 0.089 & 0.089 & 0.096 & 0.097 & 0.097 \\
\textbf{10} & cover & 0.047 & 0.051 & 0.051 & 0.068 & 0.060 & 0.060 \\
\textbf{11} & donors & 0.071 & 0.079 & 0.079 & 0.072 & 0.076 & 0.076 \\
\textbf{12} & fault & 0.477 & 0.475 & 0.475 & 0.495 & 0.487 & 0.487 \\
\textbf{13} & fraud & 0.299 & 0.323 & 0.323 & 0.329 & 0.283 & 0.283 \\
\textbf{14} & glass & 0.083 & 0.094 & 0.094 & 0.104 & 0.133 & 0.133 \\
\textbf{15} & Hepatitis & 0.454 & 0.459 & 0.459 & 0.242 & 0.250 & 0.250 \\
\textbf{16} & http & 0.296 & 0.309 & 0.309 & 0.285 & 0.289 & 0.289 \\
\textbf{17} & InternetAds & 0.365 & 0.405 & 0.405 & 0.336 & 0.343 & 0.343 \\
\textbf{18} & Ionosphere & 0.801 & 0.839 & 0.839 & 0.818 & 0.812 & 0.812 \\
\textbf{19} & landsat & 0.202 & 0.208 & 0.219 & 0.192 & 0.207 & 0.238 \\
\textbf{20} & letter & 0.120 & 0.130 & 0.130 & 0.135 & 0.132 & 0.132 \\
\textbf{21} & Lymphography & 0.782 & 0.705 & 0.705 & 0.563 & 0.846 & 0.846 \\
\textbf{22} & magic.gamma & 0.689 & 0.686 & 0.643 & 0.664 & 0.672 & 0.577 \\
\textbf{23} & mammography & 0.067 & 0.069 & 0.070 & 0.083 & 0.083 & 0.080 \\
\textbf{24} & mnist & 0.396 & 0.391 & 0.391 & 0.334 & 0.330 & 0.330 \\
\textbf{25} & musk & 1.000 & 1.000 & 1.000 & 0.537 & 0.316 & 0.316 \\
\textbf{26} & optdigits & 0.038 & 0.046 & 0.046 & 0.030 & 0.045 & 0.045 \\
\textbf{27} & PageBlocks & 0.553 & 0.584 & 0.584 & 0.543 & 0.543 & 0.457 \\
\textbf{28} & pendigits & 0.253 & 0.335 & 0.335 & 0.173 & 0.187 & 0.187 \\
\textbf{29} & Pima & 0.498 & 0.494 & 0.494 & 0.501 & 0.510 & 0.510 \\
\textbf{30} & satellite & 0.681 & 0.698 & 0.698 & 0.634 & 0.670 & 0.670 \\
\textbf{31} & satimage-2 & 0.907 & 0.920 & 0.920 & 0.948 & 0.942 & 0.942 \\
\textbf{32} & shuttle & 0.959 & 0.960 & 0.972 & 0.944 & 0.949 & 0.302 \\
\textbf{33} & skin & 0.449 & 0.441 & 0.441 & 0.457 & 0.462 & 0.462 \\
\textbf{34} & smtp & 0.079 & 0.033 & 0.033 & 0.507 & 0.409 & 0.409 \\
\textbf{35} & SpamBase & 0.380 & 0.389 & 0.528 & 0.394 & 0.392 & 0.512 \\
\textbf{36} & speech & 0.017 & 0.017 & 0.017 & 0.018 & 0.019 & 0.019 \\
\textbf{37} & Stamps & 0.309 & 0.312 & 0.312 & 0.338 & 0.331 & 0.331 \\
\textbf{38} & thyroid & 0.228 & 0.220 & 0.220 & 0.244 & 0.253 & 0.253 \\
\textbf{39} & vertebral & 0.097 & 0.095 & 0.095 & 0.097 & 0.094 & 0.094 \\
\textbf{40} & vowels & 0.308 & 0.297 & 0.297 & 0.242 & 0.290 & 0.290 \\
\textbf{41} & Waveform & 0.129 & 0.113 & 0.113 & 0.180 & 0.217 & 0.217 \\
\textbf{42} & WBC & 0.862 & 0.869 & 0.869 & 0.859 & 0.876 & 0.876 \\
\textbf{43} & WDBC & 0.625 & 0.530 & 0.530 & 0.783 & 0.568 & 0.568 \\
\textbf{44} & Wilt & 0.037 & 0.039 & 0.042 & 0.038 & 0.039 & 0.042 \\
\textbf{45} & wine & 0.502 & 0.416 & 0.416 & 0.259 & 0.233 & 0.233 \\
\textbf{46} & WPBC & 0.230 & 0.233 & 0.233 & 0.230 & 0.242 & 0.242 \\
\textbf{47} & yeast & 0.300 & 0.299 & 0.331 & 0.291 & 0.298 & 0.337 \\
\textbf{48} & CIFAR10 & 0.526 & 0.525 & 0.525 & 0.378 & 0.376 & 0.376 \\
\textbf{49} & FashionMNIST & 0.695 & 0.697 & 0.697 & 0.473 & 0.473 & 0.473 \\
\textbf{50} & MNIST-C & 0.304 & 0.313 & 0.313 & 0.213 & 0.215 & 0.215 \\
\textbf{51} & MVTec-AD & 0.535 & 0.534 & 0.534 & 0.510 & 0.503 & 0.503 \\
\textbf{52} & SVHN & 0.065 & 0.065 & 0.065 & 0.065 & 0.065 & 0.065 \\
\textbf{53} & Agnews & 0.281 & 0.264 & 0.264 & 0.286 & 0.285 & 0.285 \\
\textbf{54} & Amazon & 0.054 & 0.055 & 0.055 & 0.054 & 0.054 & 0.054 \\
\textbf{55} & Imdb & 0.055 & 0.056 & 0.056 & 0.055 & 0.055 & 0.055 \\
\textbf{56} & Yelp & 0.057 & 0.057 & 0.057 & 0.057 & 0.057 & 0.057 \\
\textbf{57} & 20news & 0.117 & 0.119 & 0.119 & 0.129 & 0.129 & 0.129 \\
\midrule
\multicolumn{2}{c}{\textbf{Average}} & 0.350 & 0.349 & \textbf{0.352} & \textbf{0.334} & 0.332 & 0.322 \\
\bottomrule
\end{tabular}
}
\end{table}


\begin{sidewaystable}[h!]
\centering
\caption{Per-dataset AUROC comparison on the 57 \texttt{ADBench} datasets. 
Methods marked with `*' are implemented and evaluated in our framework, while the remaining baseline results are taken from \citep{DBLP:journals/corr/abs-2305-18593}.
Best average AUROC for each method is \textbf{bolded}.}
\label{tab:adbench1}
\setlength{\tabcolsep}{0.7mm} 
\fontsize{6pt}{7pt}\selectfont 
\begin{tabular}{l|cccccccccccccccccccccccc|c}
\toprule
    & CBLOF & COPOD & ECOD  & FeatureBagging & HBOS  & IForest & kNN   & LODA  & LOF   & MCD   & OCSVM & PCA   & DAGMM & DeepSVDD & DROCC & GOAD  & ICL   & PlanarFlow & DDPM  & DTE-NP & DTE-IG & DTE-C & ODIM*  & ALTBI* & Ours* (w/ ALTBI) \\
\midrule
aloi             & 0.556 & 0.515 & 0.531 & 0.792          & 0.531 & 0.542   & 0.613 & 0.495 & 0.767 & 0.520 & 0.549 & 0.549 & 0.517 & 0.514    & 0.500 & 0.497 & 0.548 & 0.520      & 0.532 & 0.645  & 0.541  & 0.525 & 0.525 &	0.521 	& 0.601     \\
annthyroid       & 0.676 & 0.777 & 0.789 & 0.788          & 0.608 & 0.816   & 0.761 & 0.453 & 0.710 & 0.918 & 0.682 & 0.676 & 0.548 & 0.739    & 0.631 & 0.453 & 0.599 & 0.966      & 0.814 & 0.781  & 0.923  & 0.964 & 0.632 &	0.633 &	0.691 
    \\
backdoor         & 0.897 & 0.500 & 0.500 & 0.790          & 0.740 & 0.725   & 0.826 & 0.515 & 0.764 & 0.848 & 0.889 & 0.888 & 0.752 & 0.735    & 0.500 & 0.587 & 0.936 & 0.787      & 0.892 & 0.806  & 0.753  & 0.875 & 0.909 &	0.876 &	0.868 
    \\
breastw          & 0.961 & 0.994 & 0.990 & 0.408          & 0.984 & 0.983   & 0.980 & 0.970 & 0.446 & 0.985 & 0.935 & 0.946 & 0.811 & 0.625    & 0.847 & 0.845 & 0.807 & 0.965      & 0.766 & 0.976  & 0.905  & 0.891 & 0.980 &	0.981 &	0.982 
    \\
campaign         & 0.738 & 0.783 & 0.769 & 0.594          & 0.768 & 0.704   & 0.750 & 0.493 & 0.614 & 0.775 & 0.737 & 0.734 & 0.581 & 0.508    & 0.500 & 0.443 & 0.766 & 0.566      & 0.724 & 0.746  & 0.660  & 0.789 & 0.729 &	0.726 &	0.726 
    \\
cardio           & 0.832 & 0.921 & 0.935 & 0.579          & 0.839 & 0.922   & 0.830 & 0.856 & 0.551 & 0.815 & 0.934 & 0.949 & 0.625 & 0.498    & 0.655 & 0.908 & 0.461 & 0.796      & 0.723 & 0.777  & 0.631  & 0.721 & 0.869 &	0.845 &	0.840 
    \\
cardiotocography & 0.561 & 0.664 & 0.784 & 0.538          & 0.595 & 0.681   & 0.503 & 0.708 & 0.527 & 0.500 & 0.691 & 0.747 & 0.546 & 0.488    & 0.449 & 0.624 & 0.372 & 0.643      & 0.579 & 0.493  & 0.506  & 0.510 & 0.542 &	0.540 &	0.530 
    \\
celeba           & 0.753 & 0.757 & 0.763 & 0.514          & 0.754 & 0.707   & 0.736 & 0.600 & 0.432 & 0.803 & 0.781 & 0.792 & 0.627 & 0.491    & 0.726 & 0.432 & 0.684 & 0.703      & 0.796 & 0.699  & 0.700  & 0.812 & 0.782 &	0.804 &	0.797 
    \\
census           & 0.664 & 0.500 & 0.500 & 0.538          & 0.611 & 0.607   & 0.671 & 0.454 & 0.562 & 0.731 & 0.655 & 0.662 & 0.491 & 0.527    & 0.443 & 0.488 & 0.668 & 0.604      & 0.659 & 0.672  & 0.629  & 0.646 & 0.682 &	0.670 &	0.672 
    \\
cover            & 0.922 & 0.882 & 0.919 & 0.571          & 0.707 & 0.873   & 0.866 & 0.922 & 0.568 & 0.696 & 0.952 & 0.934 & 0.742 & 0.580    & 0.747 & 0.124 & 0.681 & 0.417      & 0.808 & 0.838  & 0.635  & 0.697 & 0.920 &	0.897 &	0.910 
    \\
donors           & 0.808 & 0.815 & 0.888 & 0.691          & 0.743 & 0.771   & 0.829 & 0.566 & 0.629 & 0.765 & 0.770 & 0.825 & 0.558 & 0.511    & 0.747 & 0.225 & 0.739 & 0.899      & 0.806 & 0.832  & 0.796  & 0.785 & 0.606 &	0.609 &	0.650 
    \\
fault            & 0.665 & 0.455 & 0.468 & 0.591          & 0.506 & 0.544   & 0.715 & 0.478 & 0.579 & 0.505 & 0.537 & 0.480 & 0.495 & 0.522    & 0.668 & 0.546 & 0.661 & 0.469      & 0.562 & 0.726  & 0.577  & 0.590 & 0.678 &	0.661 &	0.658 
    \\
fraud            & 0.954 & 0.943 & 0.949 & 0.616          & 0.945 & 0.950   & 0.955 & 0.856 & 0.548 & 0.911 & 0.954 & 0.952 & 0.857 & 0.769    & 0.500 & 0.724 & 0.931 & 0.895      & 0.924 & 0.956  & 0.942  & 0.938 & 0.913 &	0.920 &	0.925 
    \\
glass            & 0.855 & 0.760 & 0.710 & 0.659          & 0.820 & 0.790   & 0.870 & 0.624 & 0.618 & 0.795 & 0.661 & 0.715 & 0.630 & 0.517    & 0.743 & 0.545 & 0.729 & 0.766      & 0.560 & 0.881  & 0.681  & 0.864 & 0.757 &	0.748 &	0.771 
    \\
hepatitis        & 0.635 & 0.807 & 0.737 & 0.469          & 0.768 & 0.683   & 0.669 & 0.557 & 0.468 & 0.721 & 0.704 & 0.748 & 0.600 & 0.361    & 0.582 & 0.637 & 0.616 & 0.654      & 0.461 & 0.631  & 0.451  & 0.577 & 0.703 &	0.811 &	0.818 
    \\
http             & 0.996 & 0.991 & 0.980 & 0.288          & 0.991 & 1.000   & 0.051 & 0.060 & 0.338 & 1.000 & 0.994 & 0.997 & 0.838 & 0.249    & 0.500 & 0.996 & 0.921 & 0.994      & 0.998 & 0.051  & 0.973  & 0.995 & 0.996 &	0.996 &	0.996 
    \\
internetads      & 0.616 & 0.676 & 0.677 & 0.494          & 0.696 & 0.686   & 0.616 & 0.541 & 0.587 & 0.660 & 0.615 & 0.609 & 0.515 & 0.583    & 0.500 & 0.614 & 0.592 & 0.608      & 0.614 & 0.634  & 0.635  & 0.656 & 0.731 &	0.733 &	0.745 
    \\
ionosphere       & 0.892 & 0.783 & 0.717 & 0.876          & 0.544 & 0.833   & 0.922 & 0.788 & 0.864 & 0.951 & 0.838 & 0.777 & 0.641 & 0.514    & 0.766 & 0.829 & 0.629 & 0.884      & 0.758 & 0.924  & 0.697  & 0.911 & 0.862 &	0.850 &	0.883 
    \\
landsat          & 0.548 & 0.422 & 0.368 & 0.540          & 0.575 & 0.474   & 0.614 & 0.382 & 0.549 & 0.607 & 0.423 & 0.366 & 0.533 & 0.631    & 0.626 & 0.506 & 0.649 & 0.464      & 0.496 & 0.602  & 0.473  & 0.544 & 0.499 	&0.527 &	0.569 
    \\
letter           & 0.763 & 0.560 & 0.573 & 0.886          & 0.589 & 0.616   & 0.812 & 0.537 & 0.878 & 0.804 & 0.598 & 0.524 & 0.503 & 0.517    & 0.780 & 0.598 & 0.737 & 0.689      & 0.847 & 0.850  & 0.676  & 0.781 & 0.730 &	0.722 &	0.733 
    \\
lymphography     & 0.994 & 0.996 & 0.995 & 0.523          & 0.995 & 0.999   & 0.995 & 0.900 & 0.636 & 0.989 & 0.996 & 0.997 & 0.840 & 0.681    & 0.878 & 0.995 & 0.884 & 0.940      & 0.958 & 0.989  & 0.852  & 0.834 & 0.980 &	0.988 &	0.988 
    \\
magic.gamma      & 0.725 & 0.681 & 0.638 & 0.700          & 0.709 & 0.721   & 0.795 & 0.655 & 0.678 & 0.699 & 0.673 & 0.667 & 0.584 & 0.604    & 0.728 & 0.442 & 0.676 & 0.742      & 0.763 & 0.801  & 0.782  & 0.765 & 0.724 &	0.746 &	0.694 
    \\
mammography      & 0.795 & 0.905 & 0.906 & 0.726          & 0.838 & 0.860   & 0.852 & 0.867 & 0.702 & 0.690 & 0.871 & 0.888 & 0.719 & 0.451    & 0.779 & 0.414 & 0.658 & 0.782      & 0.749 & 0.849  & 0.799  & 0.810 & 0.813 &	0.808 &	0.763 
    \\

musk             & 1.000 & 0.948 & 0.953 & 0.575          & 1.000 & 0.998   & 0.964 & 0.993 & 0.581 & 1.000 & 1.000 & 1.000 & 0.912 & 0.538    & 0.575 & 1.000 & 0.790 & 0.748      & 1.000 & 0.882  & 0.785  & 0.965 & 0.971 &	1.000 &	1.000 
    \\
optdigits        & 0.785 & 0.500 & 0.500 & 0.539          & 0.868 & 0.696   & 0.395 & 0.493 & 0.538 & 0.413 & 0.507 & 0.518 & 0.408 & 0.519    & 0.565 & 0.657 & 0.533 & 0.492      & 0.402 & 0.386  & 0.513  & 0.508 & 0.552 &	0.602 &	0.684 
    \\
pageblocks       & 0.893 & 0.875 & 0.914 & 0.758          & 0.779 & 0.897   & 0.919 & 0.712 & 0.703 & 0.923 & 0.914 & 0.907 & 0.753 & 0.592    & 0.914 & 0.609 & 0.768 & 0.908      & 0.820 & 0.906  & 0.850  & 0.924 & 0.864 &	0.868 &	0.807 
    \\
pendigits        & 0.864 & 0.906 & 0.927 & 0.518          & 0.925 & 0.947   & 0.828 & 0.895 & 0.534 & 0.834 & 0.929 & 0.936 & 0.548 & 0.383    & 0.520 & 0.592 & 0.650 & 0.780      & 0.700 & 0.786  & 0.624  & 0.713 & 0.911 &	0.939 &	0.957 
    \\
pima             & 0.655 & 0.662 & 0.604 & 0.573          & 0.704 & 0.674   & 0.723 & 0.595 & 0.563 & 0.686 & 0.631 & 0.651 & 0.522 & 0.510    & 0.542 & 0.606 & 0.524 & 0.615      & 0.537 & 0.707  & 0.599  & 0.624 & 0.712 &	0.716 &	0.713 
    \\
satellite        & 0.742 & 0.633 & 0.583 & 0.545          & 0.762 & 0.695   & 0.721 & 0.614 & 0.550 & 0.804 & 0.662 & 0.601 & 0.675 & 0.562    & 0.608 & 0.702 & 0.627 & 0.671      & 0.715 & 0.702  & 0.582  & 0.711 & 0.702 &	0.734 &	0.760 
    \\
satimage-2       & 0.999 & 0.975 & 0.965 & 0.526          & 0.976 & 0.993   & 0.992 & 0.981 & 0.539 & 0.995 & 0.997 & 0.977 & 0.911 & 0.551    & 0.579 & 0.996 & 0.898 & 0.970      & 0.996 & 0.980  & 0.858  & 0.946 & 0.997 &	0.997 &	0.998 
    \\
shuttle          & 0.621 & 0.995 & 0.993 & 0.493          & 0.986 & 0.997   & 0.732 & 0.389 & 0.526 & 0.990 & 0.992 & 0.990 & 0.898 & 0.576    & 0.500 & 0.208 & 0.642 & 0.852      & 0.975 & 0.698  & 0.669  & 0.976 & 0.982 &	0.983 &	0.985 
    \\
skin             & 0.675 & 0.471 & 0.490 & 0.534          & 0.588 & 0.670   & 0.720 & 0.442 & 0.550 & 0.892 & 0.547 & 0.447 & 0.554 & 0.548    & 0.708 & 0.579 & 0.265 & 0.773      & 0.461 & 0.718  & 0.741  & 0.741 & 0.873 &	0.869 &	0.865 
    \\
smtp             & 0.863 & 0.912 & 0.882 & 0.794          & 0.809 & 0.905   & 0.933 & 0.819 & 0.899 & 0.948 & 0.845 & 0.856 & 0.868 & 0.895    & 0.500 & 0.915 & 0.656 & 0.784      & 0.956 & 0.930  & 0.769  & 0.951 & 0.804 &	0.823 &	0.845 
    \\
spambase         & 0.541 & 0.688 & 0.656 & 0.424          & 0.664 & 0.637   & 0.566 & 0.480 & 0.453 & 0.446 & 0.534 & 0.548 & 0.488 & 0.584    & 0.490 & 0.496 & 0.459 & 0.528      & 0.510 & 0.545  & 0.509  & 0.515 & 0.512 &	0.491 &	0.670 
    \\
speech           & 0.471 & 0.489 & 0.470 & 0.509          & 0.473 & 0.476   & 0.480 & 0.466 & 0.512 & 0.494 & 0.466 & 0.469 & 0.522 & 0.512    & 0.483 & 0.458 & 0.512 & 0.496      & 0.466 & 0.487  & 0.488  & 0.495 & 0.471 &	0.471 &	0.468 
    \\
stamps           & 0.660 & 0.929 & 0.877 & 0.502          & 0.904 & 0.907   & 0.870 & 0.831 & 0.512 & 0.838 & 0.882 & 0.909 & 0.719 & 0.465    & 0.760 & 0.774 & 0.505 & 0.838      & 0.556 & 0.820  & 0.692  & 0.753 & 0.898 &	0.878 &	0.883 
    \\
thyroid          & 0.909 & 0.939 & 0.977 & 0.707          & 0.948 & 0.979   & 0.965 & 0.819 & 0.657 & 0.986 & 0.958 & 0.955 & 0.719 & 0.505    & 0.889 & 0.574 & 0.693 & 0.992      & 0.871 & 0.964  & 0.828  & 0.990 & 0.938 &	0.938 &	0.939 
   \\
vertebral        & 0.463 & 0.263 & 0.417 & 0.473          & 0.317 & 0.362   & 0.379 & 0.294 & 0.487 & 0.389 & 0.426 & 0.378 & 0.470 & 0.394    & 0.425 & 0.468 & 0.449 & 0.409      & 0.563 & 0.400  & 0.451  & 0.458 & 0.374 &	0.385& 	0.373 
    \\
vowels           & 0.884 & 0.496 & 0.593 & 0.933          & 0.679 & 0.763   & 0.951 & 0.705 & 0.932 & 0.732 & 0.779 & 0.604 & 0.464 & 0.514    & 0.738 & 0.791 & 0.784 & 0.888      & 0.903 & 0.964  & 0.705  & 0.914 & 0.862 &	0.878& 	0.882 
    \\
waveform         & 0.701 & 0.739 & 0.603 & 0.715          & 0.694 & 0.707   & 0.750 & 0.594 & 0.693 & 0.572 & 0.669 & 0.635 & 0.523 & 0.609    & 0.674 & 0.592 & 0.661 & 0.640      & 0.617 & 0.729  & 0.523  & 0.602 & 0.763 &	0.749 &	0.713 
    \\
wbc              & 0.977 & 0.994 & 0.994 & 0.388          & 0.987 & 0.996   & 0.982 & 0.992 & 0.607 & 0.988 & 0.987 & 0.993 & 0.821 & 0.503    & 0.821 & 0.949 & 0.853 & 0.934      & 0.948 & 0.979  & 0.894  & 0.871 & 0.992 &	0.994 &	0.993 
    \\
wdbc             & 0.990 & 0.993 & 0.971 & 0.867          & 0.989 & 0.988   & 0.980 & 0.980 & 0.849 & 0.969 & 0.984 & 0.988 & 0.715 & 0.602    & 0.347 & 0.983 & 0.738 & 0.985      & 0.965 & 0.975  & 0.566  & 0.835 & 0.993 &	0.985 &	0.981 
    \\
wilt             & 0.396 & 0.345 & 0.394 & 0.666          & 0.348 & 0.451   & 0.511 & 0.313 & 0.678 & 0.859 & 0.317 & 0.239 & 0.432 & 0.465    & 0.400 & 0.555 & 0.649 & 0.794      & 0.659 & 0.552  & 0.834  & 0.851 & 0.364 &	0.349 &	0.409 
    \\
wine             & 0.453 & 0.865 & 0.738 & 0.323          & 0.907 & 0.786   & 0.470 & 0.822 & 0.330 & 0.975 & 0.671 & 0.819 & 0.513 & 0.507    & 0.621 & 0.734 & 0.455 & 0.390      & 0.374 & 0.425  & 0.310  & 0.557 & 0.875 &	0.948 &	0.940 
   \\
wpbc             & 0.487 & 0.519 & 0.489 & 0.436          & 0.548 & 0.516   & 0.512 & 0.501 & 0.447 & 0.534 & 0.485 & 0.486 & 0.449 & 0.493    & 0.483 & 0.466 & 0.488 & 0.483      & 0.493 & 0.502  & 0.489  & 0.468 & 0.502 &	0.506 &	0.503 
    \\
yeast            & 0.461 & 0.380 & 0.443 & 0.465          & 0.402 & 0.394   & 0.396 & 0.461 & 0.453 & 0.406 & 0.420 & 0.418 & 0.503 & 0.520    & 0.396 & 0.503 & 0.466 & 0.442      & 0.463 & 0.400  & 0.446  & 0.420 & 0.398 &	0.411 &	0.469 
    \\
CIFAR10          & 0.663 & 0.548 & 0.567 & 0.687          & 0.572 & 0.629   & 0.659 & 0.591 & 0.686 & 0.639 & 0.663 & 0.659 & 0.530 & 0.555    & 0.503 & 0.659 & 0.557 & 0.621      & 0.663 & 0.660  & 0.595  & 0.629 & 0.877 &	0.889 &	0.888 
    \\
MNIST-C          & 0.757 & 0.500 & 0.500 & 0.702          & 0.689 & 0.733   & 0.786 & 0.591 & 0.699 & 0.739 & 0.751 & 0.741 & 0.581 & 0.552    & 0.594 & 0.752 & 0.670 & 0.705      & 0.751 & 0.788  & 0.703  & 0.746 & 0.726 &	0.739 &	0.741 
   \\
MVTec-AD         & 0.754 & 0.500 & 0.500 & 0.745          & 0.732 & 0.747   & 0.763 & 0.644 & 0.742 & 0.618 & 0.735 & 0.724 & 0.596 & 0.603    & 0.544 & 0.730 & 0.683 & 0.637      & 0.732 & 0.761  & 0.655  & 0.730 & 0.849 &	0.867 &	0.860 
    \\
SVHN             & 0.601 & 0.500 & 0.500 & 0.629          & 0.542 & 0.580   & 0.604 & 0.534 & 0.628 & 0.583 & 0.604 & 0.599 & 0.528 & 0.521    & 0.521 & 0.597 & 0.571 & 0.580      & 0.605 & 0.607  & 0.567  & 0.600 & 0.579 &	0.581 &	0.581 
    \\
mnist            & 0.843 & 0.500 & 0.500 & 0.664          & 0.574 & 0.811   & 0.867 & 0.564 & 0.658 & 0.856 & 0.849 & 0.848 & 0.631 & 0.605    & 0.615 & 0.698 & 0.691 & 0.645      & 0.816 & 0.853  & 0.756  & 0.819 & 0.791 &	0.796 &	0.794    \\
FashionMNIST     & 0.871 & 0.500 & 0.500 & 0.748          & 0.748 & 0.831   & 0.875 & 0.672 & 0.738 & 0.840 & 0.860 & 0.853 & 0.664 & 0.647    & 0.564 & 0.860 & 0.758 & 0.819      & 0.861 & 0.873  & 0.767  & 0.841 & 0.902 &	0.921 &	0.920    \\
20news           & 0.564 & 0.533 & 0.544 & 0.610          & 0.537 & 0.550   & 0.567 & 0.539 & 0.610 & 0.583 & 0.559 & 0.545 & 0.518 & 0.515    & 0.496 & 0.553 & 0.547 & 0.513      & 0.547 & 0.570  & 0.527  & 0.579 & 0.725 &	0.725 &	0.725 
   \\
agnews           & 0.619 & 0.551 & 0.552 & 0.715          & 0.554 & 0.584   & 0.647 & 0.568 & 0.714 & 0.665 & 0.601 & 0.566 & 0.508 & 0.494    & 0.500 & 0.592 & 0.591 & 0.497      & 0.571 & 0.652  & 0.545  & 0.627 &0.862 &	0.868 &	0.864 
    \\
amazon           & 0.579 & 0.571 & 0.541 & 0.572          & 0.563 & 0.558   & 0.603 & 0.526 & 0.571 & 0.597 & 0.565 & 0.550 & 0.501 & 0.464    & 0.500 & 0.560 & 0.528 & 0.495      & 0.551 & 0.603  & 0.535  & 0.556 & 0.544 &	0.546 &	0.548 
    \\
imdb             & 0.496 & 0.512 & 0.471 & 0.499          & 0.499 & 0.489   & 0.494 & 0.466 & 0.500 & 0.504 & 0.484 & 0.478 & 0.487 & 0.526    & 0.500 & 0.486 & 0.521 & 0.493      & 0.478 & 0.495  & 0.486  & 0.484 & 0.524 &	0.523 &	0.527 
    \\
yelp             & 0.635 & 0.605 & 0.578 & 0.661          & 0.600 & 0.602   & 0.670 & 0.581 & 0.661 & 0.655 & 0.621 & 0.592 & 0.498 & 0.524    & 0.504 & 0.590 & 0.545 & 0.527      & 0.594 & 0.671  & 0.514  & 0.602 & 0.558 &	0.555 &	0.557 
    \\
\midrule
\textbf{Average} & 0.731 & 0.692 & 0.689 & 0.608          & 0.714 & 0.737   & 0.728 & 0.627 & 0.611 & 0.748 & 0.725 & 0.718 & 0.614 & 0.543    & 0.600 & 0.627 & 0.643 & 0.696      & 0.700 & 0.722  & 0.660  & 0.728 & 0.751 &	0.757 &	\textbf{0.766} 
    \\ 
\bottomrule
\end{tabular}
\end{sidewaystable}

\begin{sidewaystable}[h!]
\centering
\caption{Per-dataset AUPRC comparison on the 57 \texttt{ADBench} datasets. 
Methods marked with `*' are implemented and evaluated in our framework, while the remaining baseline results are taken from \citep{DBLP:journals/corr/abs-2305-18593}.
Best average AUPRC for each method is \textbf{bolded}.}
\label{tab:adbench2}
\setlength{\tabcolsep}{0.7mm} 
\fontsize{6pt}{7pt}\selectfont 
\begin{tabular}{l|cccccccccccccccccccccccc|c}
\toprule
    & CBLOF & COPOD & ECOD  & FeatureBagging & HBOS  & IForest & kNN   & LODA  & LOF   & MCD   & OCSVM & PCA   & DAGMM & DeepSVDD & DROCC & GOAD  & ICL   & PlanarFlow & DDPM  & DTE-NP & DTE-IG & DTE-C & ODIM*  & ALTBI* & Ours* (w/ ALTBI) \\ 
\midrule
aloi             & 0.037 & 0.031 & 0.033 & 0.104          & 0.034 & 0.034   & 0.048 & 0.033 & 0.097 & 0.032 & 0.039 & 0.037 & 0.033 & 0.034    & 0.030 & 0.033 & 0.046 & 0.032      & 0.036 & 0.056  & 0.040  & 0.033 & 0.037 &	0.035 &	0.057 
   \\
annthyroid       & 0.169 & 0.174 & 0.272 & 0.206          & 0.228 & 0.312   & 0.224 & 0.098 & 0.163 & 0.503 & 0.188 & 0.196 & 0.109 & 0.192    & 0.186 & 0.131 & 0.123 & 0.654      & 0.297 & 0.228  & 0.380  & 0.670 & 0.154 &	0.156 &	0.181 
   \\
backdoor         & 0.547 & 0.025 & 0.025 & 0.217          & 0.052 & 0.045   & 0.479 & 0.101 & 0.358 & 0.122 & 0.534 & 0.531 & 0.250 & 0.372    & 0.025 & 0.347 & 0.717 & 0.336      & 0.520 & 0.473  & 0.438  & 0.481 & 0.406 &	0.138 &	0.122 
   \\
breastw          & 0.890 & 0.989 & 0.982 & 0.284          & 0.954 & 0.956   & 0.932 & 0.955 & 0.297 & 0.962 & 0.897 & 0.946 & 0.660 & 0.482    & 0.776 & 0.826 & 0.635 & 0.908      & 0.537 & 0.921  & 0.770  & 0.715 & 0.930 &	0.943 &	0.952 
   \\
campaign         & 0.287 & 0.368 & 0.354 & 0.145          & 0.352 & 0.279   & 0.289 & 0.131 & 0.158 & 0.325 & 0.283 & 0.284 & 0.163 & 0.149    & 0.113 & 0.105 & 0.267 & 0.191      & 0.299 & 0.281  & 0.237  & 0.321 & 0.246 &	0.266 &	0.247 
   \\
cardio           & 0.482 & 0.576 & 0.567 & 0.161          & 0.458 & 0.559   & 0.402 & 0.428 & 0.159 & 0.364 & 0.536 & 0.609 & 0.193 & 0.177    & 0.272 & 0.540 & 0.108 & 0.471      & 0.278 & 0.376  & 0.184  & 0.268 & 0.526 &	0.539 &	0.519 
   \\
cardiotocography & 0.335 & 0.403 & 0.502 & 0.276          & 0.361 & 0.436   & 0.324 & 0.463 & 0.272 & 0.311 & 0.408 & 0.462 & 0.271 & 0.252    & 0.258 & 0.403 & 0.188 & 0.348      & 0.338 & 0.312  & 0.250  & 0.276 & 0.390 &	0.424 &	0.409 
   \\
celeba           & 0.069 & 0.093 & 0.095 & 0.024          & 0.090 & 0.063   & 0.061 & 0.047 & 0.018 & 0.092 & 0.103 & 0.112 & 0.044 & 0.031    & 0.047 & 0.021 & 0.045 & 0.066      & 0.093 & 0.052  & 0.058  & 0.077 & 0.064 &	0.098 &	0.090 
   \\
census           & 0.088 & 0.062 & 0.062 & 0.061          & 0.073 & 0.073   & 0.088 & 0.065 & 0.069 & 0.153 & 0.085 & 0.087 & 0.062 & 0.075    & 0.058 & 0.072 & 0.095 & 0.074      & 0.086 & 0.090  & 0.083  & 0.081 & 0.096 &	0.090 &	0.089 
   \\
cover            & 0.070 & 0.068 & 0.113 & 0.019          & 0.026 & 0.052   & 0.054 & 0.090 & 0.019 & 0.016 & 0.099 & 0.075 & 0.044 & 0.048    & 0.056 & 0.005 & 0.022 & 0.010      & 0.046 & 0.048  & 0.025  & 0.021 & 0.068 &	0.047 &	0.051 
   \\
donors           & 0.148 & 0.209 & 0.265 & 0.120          & 0.135 & 0.124   & 0.182 & 0.255 & 0.109 & 0.141 & 0.139 & 0.166 & 0.086 & 0.112    & 0.123 & 0.040 & 0.119 & 0.241      & 0.143 & 0.188  & 0.164  & 0.140 & 0.072 &	0.071 &	0.079 
   \\
fault            & 0.473 & 0.313 & 0.325 & 0.396          & 0.360 & 0.395   & 0.522 & 0.337 & 0.388 & 0.334 & 0.401 & 0.332 & 0.361 & 0.375    & 0.496 & 0.381 & 0.473 & 0.329      & 0.392 & 0.532  & 0.417  & 0.422 & 0.495 &	0.477 &	0.475 
   \\
fraud            & 0.145 & 0.252 & 0.215 & 0.003          & 0.209 & 0.145   & 0.169 & 0.146 & 0.003 & 0.488 & 0.110 & 0.149 & 0.084 & 0.250    & 0.002 & 0.257 & 0.127 & 0.447      & 0.146 & 0.137  & 0.188  & 0.648 & 0.329 &	0.299 &	0.323 
   \\
glass            & 0.144 & 0.111 & 0.183 & 0.151          & 0.161 & 0.144   & 0.167 & 0.090 & 0.144 & 0.113 & 0.130 & 0.112 & 0.111 & 0.090    & 0.159 & 0.076 & 0.122 & 0.113      & 0.073 & 0.206  & 0.135  & 0.168 & 0.104 &	0.083 &	0.094 
   \\
hepatitis        & 0.304 & 0.389 & 0.295 & 0.225          & 0.328 & 0.243   & 0.252 & 0.275 & 0.214 & 0.363 & 0.277 & 0.339 & 0.253 & 0.170    & 0.221 & 0.291 & 0.231 & 0.317      & 0.165 & 0.238  & 0.215  & 0.257 & 0.242 &	0.454 &	0.459 
   \\
http             & 0.464 & 0.280 & 0.145 & 0.047          & 0.302 & 0.886   & 0.010 & 0.004 & 0.050 & 0.865 & 0.356 & 0.500 & 0.368 & 0.093    & 0.004 & 0.441 & 0.091 & 0.363      & 0.642 & 0.024  & 0.295  & 0.440 & 0.285 &	0.296 &	0.309 
   \\
internetads      & 0.297 & 0.505 & 0.505 & 0.182          & 0.523 & 0.486   & 0.296 & 0.242 & 0.232 & 0.344 & 0.291 & 0.276 & 0.207 & 0.252    & 0.197 & 0.288 & 0.237 & 0.262      & 0.295 & 0.290  & 0.275  & 0.302 & 0.336 &	0.365 &	0.405 
   \\
ionosphere       & 0.881 & 0.663 & 0.633 & 0.821          & 0.353 & 0.779   & 0.911 & 0.741 & 0.807 & 0.947 & 0.829 & 0.721 & 0.473 & 0.392    & 0.728 & 0.781 & 0.472 & 0.824      & 0.633 & 0.920  & 0.610  & 0.880 & 0.818 &	0.801 &	0.839 
  \\
landsat          & 0.212 & 0.176 & 0.164 & 0.246          & 0.231 & 0.194   & 0.258 & 0.183 & 0.250 & 0.253 & 0.175 & 0.163 & 0.230 & 0.362    & 0.272 & 0.198 & 0.329 & 0.187      & 0.200 & 0.255  & 0.203  & 0.223 & 0.192 &	0.202 &	0.219 
  \\
letter           & 0.166 & 0.068 & 0.077 & 0.445          & 0.078 & 0.086   & 0.203 & 0.083 & 0.433 & 0.174 & 0.113 & 0.076 & 0.083 & 0.099    & 0.252 & 0.099 & 0.208 & 0.153      & 0.367 & 0.255  & 0.181  & 0.257 & 0.135 &	0.120 &	0.130 
  \\
lymphography     & 0.915 & 0.907 & 0.894 & 0.090          & 0.919 & 0.972   & 0.894 & 0.491 & 0.135 & 0.767 & 0.885 & 0.935 & 0.454 & 0.254    & 0.463 & 0.897 & 0.264 & 0.417      & 0.731 & 0.805  & 0.388  & 0.381 & 0.563 &	0.782 &	0.705 
   \\
magic.gamma      & 0.666 & 0.588 & 0.533 & 0.539          & 0.617 & 0.638   & 0.724 & 0.579 & 0.520 & 0.632 & 0.625 & 0.589 & 0.450 & 0.499    & 0.627 & 0.326 & 0.548 & 0.692      & 0.651 & 0.730  & 0.657  & 0.664 & 0.664 &	0.689 &	0.643 
  \\
mammography      & 0.140 & 0.430 & 0.435 & 0.070          & 0.132 & 0.218   & 0.181 & 0.218 & 0.085 & 0.036 & 0.187 & 0.204 & 0.111 & 0.025    & 0.114 & 0.046 & 0.046 & 0.074      & 0.099 & 0.175  & 0.082  & 0.170 & 0.083 &	0.067 &	0.070 
   \\
musk             & 1.000 & 0.369 & 0.475 & 0.140          & 0.999 & 0.945   & 0.708 & 0.842 & 0.118 & 0.992 & 1.000 & 1.000 & 0.500 & 0.107    & 0.196 & 1.000 & 0.128 & 0.391      & 0.984 & 0.434  & 0.137  & 0.553 & 0.537 &	1.000 &	1.000 
   \\
optdigits        & 0.059 & 0.029 & 0.029 & 0.036          & 0.192 & 0.046   & 0.022 & 0.029 & 0.035 & 0.022 & 0.027 & 0.027 & 0.026 & 0.039    & 0.032 & 0.039 & 0.030 & 0.027      & 0.022 & 0.021  & 0.028  & 0.028 & 0.030 &	0.038 &	0.046 
   \\
pageblocks       & 0.547 & 0.370 & 0.520 & 0.341          & 0.319 & 0.464   & 0.556 & 0.410 & 0.292 & 0.617 & 0.531 & 0.525 & 0.255 & 0.288    & 0.632 & 0.373 & 0.285 & 0.538      & 0.493 & 0.530  & 0.507  & 0.555 & 0.543 &	0.553 &	0.584 
   \\
pendigits        & 0.192 & 0.177 & 0.270 & 0.048          & 0.247 & 0.260   & 0.100 & 0.186 & 0.040 & 0.069 & 0.226 & 0.219 & 0.056 & 0.022    & 0.027 & 0.075 & 0.045 & 0.060      & 0.056 & 0.089  & 0.044  & 0.044 & 0.173 &	0.253 &	0.335 
   \\
pima             & 0.484 & 0.536 & 0.484 & 0.412          & 0.577 & 0.510   & 0.530 & 0.404 & 0.406 & 0.498 & 0.477 & 0.492 & 0.372 & 0.366    & 0.413 & 0.476 & 0.385 & 0.476      & 0.400 & 0.528  & 0.437  & 0.447 & 0.501 &	0.498 &	0.494 
   \\
satellite        & 0.656 & 0.570 & 0.526 & 0.378          & 0.688 & 0.649   & 0.582 & 0.613 & 0.381 & 0.768 & 0.654 & 0.606 & 0.527 & 0.406    & 0.465 & 0.658 & 0.451 & 0.596      & 0.662 & 0.563  & 0.380  & 0.529 & 0.634 &	0.681 &	0.698 
  \\
satimage-2       & 0.972 & 0.797 & 0.666 & 0.042          & 0.760 & 0.918   & 0.690 & 0.857 & 0.041 & 0.682 & 0.965 & 0.872 & 0.289 & 0.052    & 0.076 & 0.949 & 0.102 & 0.484      & 0.783 & 0.507  & 0.095  & 0.138 & 0.948 &	0.907 &	0.920 
   \\
shuttle          & 0.184 & 0.962 & 0.905 & 0.081          & 0.965 & 0.976   & 0.193 & 0.168 & 0.109 & 0.841 & 0.907 & 0.913 & 0.438 & 0.149    & 0.072 & 0.136 & 0.135 & 0.346      & 0.779 & 0.187  & 0.247  & 0.626 & 0.944 &	0.959 &	0.972 
   \\
skin             & 0.289 & 0.179 & 0.183 & 0.207          & 0.232 & 0.254   & 0.290 & 0.180 & 0.221 & 0.490 & 0.220 & 0.172 & 0.226 & 0.221    & 0.285 & 0.232 & 0.173 & 0.335      & 0.175 & 0.290  & 0.316  & 0.302 & 0.457 &	0.449 &	0.441 
   \\
smtp             & 0.403 & 0.005 & 0.589 & 0.001          & 0.005 & 0.005   & 0.415 & 0.312 & 0.022 & 0.006 & 0.383 & 0.382 & 0.179 & 0.240    & 0.000 & 0.358 & 0.004 & 0.006      & 0.502 & 0.411  & 0.012  & 0.422 & 0.507 &	0.079 &	0.033 
   \\
spambase         & 0.402 & 0.544 & 0.518 & 0.344          & 0.518 & 0.488   & 0.415 & 0.387 & 0.360 & 0.349 & 0.402 & 0.409 & 0.389 & 0.456    & 0.383 & 0.387 & 0.370 & 0.433      & 0.384 & 0.407  & 0.399  & 0.400 & 0.394 &	0.380 &	0.528 
   \\
speech           & 0.019 & 0.019 & 0.020 & 0.022          & 0.023 & 0.021   & 0.019 & 0.016 & 0.022 & 0.019 & 0.019 & 0.018 & 0.022 & 0.018    & 0.020 & 0.019 & 0.020 & 0.018      & 0.020 & 0.019  & 0.019  & 0.020 & 0.018 &	0.017 &	0.017 
   \\
stamps           & 0.211 & 0.398 & 0.324 & 0.143          & 0.332 & 0.347   & 0.317 & 0.280 & 0.153 & 0.257 & 0.318 & 0.364 & 0.198 & 0.099    & 0.241 & 0.285 & 0.117 & 0.284      & 0.143 & 0.273  & 0.235  & 0.226 & 0.338 &	0.309 &	0.312 
   \\
thyroid          & 0.272 & 0.179 & 0.472 & 0.069          & 0.501 & 0.562   & 0.392 & 0.189 & 0.077 & 0.702 & 0.329 & 0.356 & 0.126 & 0.024    & 0.338 & 0.318 & 0.066 & 0.734      & 0.325 & 0.360  & 0.118  & 0.705 & 0.244 	&0.228 &	0.220 
   \\
vertebral        & 0.123 & 0.085 & 0.110 & 0.124          & 0.091 & 0.097   & 0.095 & 0.089 & 0.130 & 0.101 & 0.107 & 0.099 & 0.134 & 0.107    & 0.118 & 0.124 & 0.115 & 0.111      & 0.150 & 0.098  & 0.133  & 0.119 & 0.097 	&0.097 &	0.095 
   \\
vowels           & 0.166 & 0.034 & 0.083 & 0.314          & 0.078 & 0.162   & 0.443 & 0.127 & 0.326 & 0.085 & 0.196 & 0.069 & 0.041 & 0.037    & 0.178 & 0.154 & 0.219 & 0.295      & 0.311 & 0.504  & 0.166  & 0.417 & 0.242 &	0.308 &	0.297 
   \\
waveform         & 0.122 & 0.057 & 0.040 & 0.078          & 0.048 & 0.056   & 0.133 & 0.040 & 0.071 & 0.040 & 0.052 & 0.044 & 0.032 & 0.061    & 0.150 & 0.042 & 0.063 & 0.150      & 0.050 & 0.109  & 0.037  & 0.043 & 0.180 &	0.129 &	0.113 
  \\
wbc              & 0.691 & 0.883 & 0.882 & 0.037          & 0.728 & 0.948   & 0.743 & 0.898 & 0.077 & 0.839 & 0.813 & 0.913 & 0.327 & 0.069    & 0.358 & 0.736 & 0.211 & 0.431      & 0.758 & 0.722  & 0.348  & 0.194 & 0.859 &	0.862 &	0.869 
   \\
wdbc             & 0.688 & 0.760 & 0.493 & 0.155          & 0.761 & 0.702   & 0.521 & 0.527 & 0.128 & 0.395 & 0.539 & 0.613 & 0.152 & 0.063    & 0.039 & 0.589 & 0.065 & 0.568      & 0.483 & 0.465  & 0.074  & 0.157 & 0.783 &	0.625 &	0.530 
   \\
wilt             & 0.040 & 0.037 & 0.042 & 0.081          & 0.039 & 0.044   & 0.049 & 0.036 & 0.083 & 0.153 & 0.035 & 0.032 & 0.047 & 0.046    & 0.041 & 0.065 & 0.109 & 0.115      & 0.076 & 0.054  & 0.211  & 0.163 & 0.038 &	0.037 &	0.042 
   \\
wine             & 0.170 & 0.364 & 0.195 & 0.061          & 0.412 & 0.207   & 0.081 & 0.250 & 0.064 & 0.737 & 0.135 & 0.264 & 0.120 & 0.116    & 0.126 & 0.229 & 0.087 & 0.086      & 0.075 & 0.074  & 0.064  & 0.103 & 0.259 &	0.502 &	0.416 
   \\
wpbc             & 0.227 & 0.234 & 0.217 & 0.206          & 0.241 & 0.237   & 0.234 & 0.227 & 0.210 & 0.257 & 0.222 & 0.229 & 0.214 & 0.240    & 0.234 & 0.214 & 0.234 & 0.236      & 0.238 & 0.227  & 0.238  & 0.231 & 0.230 &	0.230 &	0.233 
  \\
yeast            & 0.314 & 0.308 & 0.332 & 0.326          & 0.328 & 0.304   & 0.294 & 0.330 & 0.315 & 0.298 & 0.303 & 0.302 & 0.353 & 0.350    & 0.284 & 0.332 & 0.318 & 0.309      & 0.320 & 0.295  & 0.306  & 0.306 & 0.291 &	0.300 &	0.331 
  \\
CIFAR10          & 0.103 & 0.065 & 0.067 & 0.115          & 0.075 & 0.089   & 0.102 & 0.086 & 0.115 & 0.084 & 0.102 & 0.101 & 0.062 & 0.073    & 0.060 & 0.102 & 0.070 & 0.085      & 0.102 & 0.104  & 0.078  & 0.092 & 0.378 &	0.526 &	0.525 
   \\
MNIST-C          & 0.173 & 0.050 & 0.050 & 0.128          & 0.126 & 0.178   & 0.191 & 0.101 & 0.127 & 0.166 & 0.179 & 0.170 & 0.092 & 0.097    & 0.096 & 0.177 & 0.098 & 0.154      & 0.178 & 0.192  & 0.141  & 0.157 & 0.213 &	0.304 &	0.313 
  \\
MVTec-AD         & 0.570 & 0.236 & 0.236 & 0.536          & 0.546 & 0.570   & 0.580 & 0.464 & 0.532 & 0.451 & 0.555 & 0.540 & 0.362 & 0.387    & 0.317 & 0.546 & 0.404 & 0.454      & 0.546 & 0.578  & 0.439  & 0.517 & 0.510 &	0.535 &	0.534 
  \\
SVHN             & 0.079 & 0.050 & 0.050 & 0.084          & 0.064 & 0.073   & 0.079 & 0.064 & 0.083 & 0.068 & 0.078 & 0.078 & 0.059 & 0.063    & 0.060 & 0.078 & 0.068 & 0.074      & 0.078 & 0.080  & 0.069  & 0.077 & 0.065 &	0.065 &	0.065 
  \\
mnist            & 0.386 & 0.092 & 0.092 & 0.241          & 0.109 & 0.290   & 0.409 & 0.170 & 0.233 & 0.308 & 0.385 & 0.381 & 0.215 & 0.253    & 0.237 & 0.297 & 0.232 & 0.259      & 0.374 & 0.400  & 0.276  & 0.368 & 0.334 &	0.396 &	0.391 
   \\
FashionMNIST     & 0.329 & 0.050 & 0.050 & 0.194          & 0.269 & 0.320   & 0.346 & 0.180 & 0.188 & 0.245 & 0.329 & 0.319 & 0.138 & 0.181    & 0.106 & 0.328 & 0.158 & 0.297      & 0.325 & 0.339  & 0.213  & 0.267 & 0.473 &	0.695 &	0.697 
   \\
20news           & 0.067 & 0.061 & 0.062 & 0.087          & 0.061 & 0.062   & 0.069 & 0.062 & 0.088 & 0.072 & 0.064 & 0.062 & 0.054 & 0.058    & 0.055 & 0.063 & 0.063 & 0.056      & 0.063 & 0.072  & 0.060  & 0.068 & 0.129 &	0.117 &	0.119 
   \\
agnews           & 0.072 & 0.059 & 0.058 & 0.125          & 0.059 & 0.064   & 0.082 & 0.064 & 0.125 & 0.077 & 0.068 & 0.061 & 0.053 & 0.053    & 0.051 & 0.066 & 0.069 & 0.050      & 0.062 & 0.085  & 0.063  & 0.076 & 0.286 &	0.281 &	0.264 
   \\
amazon           & 0.061 & 0.060 & 0.055 & 0.058          & 0.059 & 0.058   & 0.062 & 0.054 & 0.058 & 0.062 & 0.059 & 0.057 & 0.049 & 0.046    & 0.050 & 0.058 & 0.052 & 0.050      & 0.057 & 0.062  & 0.055  & 0.057 & 0.054 &	0.054 &	0.055 
   \\
imdb             & 0.047 & 0.050 & 0.045 & 0.049          & 0.047 & 0.047   & 0.047 & 0.046 & 0.049 & 0.049 & 0.047 & 0.046 & 0.049 & 0.053    & 0.050 & 0.047 & 0.054 & 0.051      & 0.046 & 0.047  & 0.047  & 0.047 & 0.055 &	0.055 &	0.056 
  \\
yelp             & 0.073 & 0.072 & 0.065 & 0.085          & 0.070 & 0.070   & 0.083 & 0.067 & 0.085 & 0.075 & 0.073 & 0.069 & 0.049 & 0.058    & 0.051 & 0.068 & 0.054 & 0.056      & 0.069 & 0.085  & 0.054  & 0.066 & 0.057 &	0.057 &	0.057 
   \\
\midrule
\textbf{Average} & 0.318 & 0.288 & 0.296 & 0.179          & 0.308 & 0.336   & 0.308 & 0.260 & 0.181 & 0.337 & 0.324 & 0.328 & 0.198 & 0.170    & 0.199 & 0.285 & 0.185 & 0.283      & 0.301 & 0.295  & 0.216  & 0.288 & 0.334 &	0.350 &	\textbf{0.352} 

\\ 
\bottomrule
\end{tabular}
\end{sidewaystable}




\clearpage
\section{Broader Impacts}
\label{sec-appen:broader_impacts}

Our paper provides a theoretical analysis of the IM effect together with a few simple modifications that improve the performance of existing IM-based outlier detection methods, and is foundational rather than tied to any specific release of generative models.
Because outlier detection can be used as a general-purpose tool, downstream uses (e.g., fraud detection, cybersecurity, or medical screening) carry the usual risks of false positives and false negatives, but these risks arise at application phase rather than from the methodology proposed here.
Beyond the considerations already common to outlier detection in such settings, we do not identify additional specific societal concerns introduced by this work.


\end{document}